\documentclass[twoside,11pt]{article}

\usepackage{amsmath,amsthm,amssymb}
\usepackage{xcolor}
\usepackage{algorithm}
\usepackage{algpseudocode}
\usepackage{subcaption}

\usepackage{amsmath,amsfonts,amssymb}
\usepackage{enumitem}
\usepackage{longtable}
\usepackage{booktabs}
\usepackage{multirow}
\usepackage{cellspace}
\usepackage{setspace}
\usepackage{bm}
\usepackage{bbm}
\usepackage{graphicx}
\usepackage{tikz,pgfplots}
\usepackage{framed}
\usepackage[utf8]{inputenc}
\usepackage[T1]{fontenc}
\usepackage{mathbbol}
\usepackage{amssymb}             

\DeclareSymbolFontAlphabet{\amsmathbb}{AMSb}%

\theoremstyle{plain}
\newtheorem{theorem}{Theorem}
\newtheorem{lemma}[theorem]{Lemma}

\newtheorem{proposition}[theorem]{Proposition}

\newtheorem*{assumption*}{\assumptionnumber}
\providecommand{\assumptionnumber}{}


\newtheorem*{theorem*}{Theorem}
\newtheorem*{lemma*}{Lemma}
\newtheorem*{corollary*}{Corollary}
\newtheorem*{proposition*}{Proposition}
\newtheorem*{claim*}{Claim}
\newtheorem*{fact*}{Fact}

\theoremstyle{definition}

\newtheorem*{definition*}{Definition}

\newtheorem*{remark*}{Remark}

\newtheorem*{example*}{Example}

\newcommand{\ignore}[1]{}

\DeclareMathOperator*{\argmin}{arg\,min}

\newcommand{\be}{\begin{align}}
\newcommand{\en}{\end{align}}

\newcommand{\ben}{\begin{align*}}
\newcommand{\enn}{\end{align*}}

\newcommand{\norm}[1]{\left\|#1\right\|}

\newcommand{\HH}{\mathcal{H}}

\newcommand{\reals}{\mathbb{R}}

\def\reals{{\mathcal R}}

\def\trace{\textup{Tr}}

\def\reals{{\mathbb R}}

\def\bold0{\mathbf{0}}

\def\be{\mathbf{e}}

\newcommand{\eps}{\varepsilon}

\usepackage{braket}

\newcommand{\sumin}{\sum_{i=1}^n}

\newcommand{\Expect}{\mathbb{E}}

\newenvironment{proofarg}[1]{%
 \renewcommand{\proofname}{Proof #1}\proof}{\endproof}

\usepackage{mathtools}

\newcommand{\startfoo}{%
    \par\medskip
    \begin{mdframed}[linewidth=1pt]%
    \let\figure\figurehere
    \let\endfigure\endfigurehere
    \ignorespaces
}
\newcommand{\stopfoo}{%
    \unskip
    \end{mdframed}%
    \par\medskip
}

\DeclareMathOperator*{\argmax}{\arg\!\max}
\newcommand{\normp}[2]{\left\|#1\right\|_{#2}}

\newcommand{\STAB}[1]{\begin{tabular}{@{}c@{}}#1\end{tabular}}

\newcommand{\du}{\mathcal{D}_{\textup{u}}}
\newcommand{\wdu}{\widehat{\mathcal{D}}_{\textup{u}}}
\newcommand{\dtrain}{\mathcal{D}_{\textup{train}}}

\newcommand{\ec}{\emph{$\epsilon$}-coreset}
\newcommand{\ecs}{\emph{$\epsilon$}-coresets}
\usepackage{jmlr2e}

\jmlrheading{1}{2021}{1-48}{9/21}{-}{borsos21}{Zal\'an Borsos, Mojm\'ir Mutn\'y, Marco Tagliasacchi and Andreas Krause}

\ShortHeadings{Data Summarization via Bilevel Optimization}{Borsos, Mutn\'y, Tagliasacchi and Krause}
\firstpageno{1}

\begin{document}

\title{Data Summarization via Bilevel Optimization}

\author{\name Zal\'an Borsos \email zalan.borsos@inf.ethz.ch \\
       \addr Department of Computer Science\\
       ETH Zurich\\
       Universit\"atstrasse 6, 8092 Zurich, Switzerland
       \AND
      \name Mojm\'ir Mutn\'y \email mojmir.mutny@inf.ethz.ch \\
       \addr Department of Computer Science\\
       ETH Zurich\\
       Universit\"atstrasse 6, 8092 Zurich, Switzerland
       \AND
      \name Marco Tagliasacchi \email mtagliasacchi@google.com\\
       \addr Google Research\\
       Brandschenkestrasse 110, 8002 Zurich, Switzerland
       \AND
      \name Andreas Krause \email krausea@ethz.ch \\
       \addr Department of Computer Science\\
       ETH Zurich\\
       Universit\"atstrasse 6, 8092 Zurich, Switzerland}

\editor{-}

\maketitle

\begin{abstract}
The increasing availability of massive data sets poses a series of challenges for machine learning. Prominent among these is the need to learn models under hardware or human resource constraints. In such resource-constrained settings, a simple yet powerful approach is to operate on small subsets of the data. Coresets are weighted subsets of the data that provide approximation guarantees for the optimization objective. However, existing coreset constructions are highly model-specific and are limited to simple models such as linear regression, logistic regression, and $k$-means.
In this work, we propose a generic coreset construction framework that formulates the coreset selection as a cardinality-constrained bilevel optimization problem. In contrast to existing approaches, our framework does not require model-specific adaptations and applies to any twice differentiable model, including neural networks. We show the effectiveness of our framework for a wide range of models in various settings, including training non-convex models online and batch active learning.
\end{abstract}


\section{Introduction}

\looseness -1 Learning models on massive data sets faces several challenges. From a computational perspective, specific hardware resource constraints must be met: the data is first loaded into the system's main memory with limited capacity, then it is processed by algorithms with possibly superlinear space or time complexity. Moreover, commonly used specialized hardware for accelerating data processing, such as GPUs, introduces another layer of constraints due to their limited memory. A simple yet powerful approach for tackling these computational challenges is to operate on small subsets of the data sampled \emph{uniformly at random}---this idea is key to the success of stochastic optimization. However, real-world settings often involve rare but essential events with a significant impact on the optimization objective that are unlikely to be represented in a small uniform summary---a drawback that can be remedied by constructing a more \emph{representative} summary.

\looseness -1 Some settings face human resource constraints. A prominent example of such a setting is  batch active learning, where the human expert is responsible for providing labels for the selected batch of points in each round of active learning.  The cost and the limited availability of human attention impose explicit constraints on the number of points that can be labeled, accentuating the importance of compact and informative summaries.

\looseness -1 \emph{Coresets} are weighted subsets of the data that provide approximation guarantees for the optimization objective---the strongest form of these are \emph{uniform} \emph{multiplicative} guarantees. On the one hand, coresets with uniform approximation guarantees are a versatile tool for obtaining provably good solutions for optimization problems on massive data sets in batch, distributed, and streaming settings.
On the other hand, such strong approximation guarantees are hard or even impossible to obtain for practical coreset sizes for some relevant problems. Consequently, coresets with uniform approximation guarantees are \emph{limited to simple models} such as linear regression, logistic regression, $k$-means, and Gaussian mixture models.

\looseness -1 Moreover, existing coreset construction strategies are either \emph{model-specific} or require model-specific derivations for instantiating their construction. For example, the popular framework of \citet{feldman2011unified} for constructing coresets with uniform approximation guarantees relies on importance sampling based on upper bounds on the sensitivity of data points; the derivation of non-vacuous sensitivity upper bounds is a difficult model-specific task. 
 While several alternative coreset definitions have been proposed, no principled coreset construction has been shown to succeed for a wide range of models. Moreover, coresets remain underexplored in practically relevant settings that rely on data summarization, including batch active learning, compression, and the training of non-convex models online, in spite  of coresets being natural candidates for these settings.

 \looseness -1  In this work, we propose a \emph{generic coreset construction framework} for twice differentiable models. We show that our method is effective for a wide range of models in various resource-constrained settings. In particular: 
\begin{itemize}\setlength\itemsep{0.5em}
    \item We formulate the coreset construction as a cardinality-constrained \emph{bilevel optimization} problem that we solve by greedy forward selection and first-order methods. In contrast to existing coreset constructions, our framework applies to \emph{any} twice differentiable model and does not require model-specific modifications. 
    \item We point out connections to \emph{robust statistics} and \emph{experimental design}, present theoretical guarantees for our framework and offer several variants for improved scalability. We discuss various extensions, including generating joint coresets for multiple models.
    \item \looseness -1  We demonstrate the advantage of our framework over other data summarization techniques in extensive experimental studies, over a wide range of models and resource-constrained settings, such as \emph{continual learning, streaming} and \emph{batch active learning}.
\end{itemize}

    \looseness -1 This work is a significant extension of \citet{NEURIPS2020_aa2a7737} that demonstrated the effectiveness of the framework for building small coresets for neural networks with proxy models only. We  extend the framework to constructing large coresets directly for the target models by proposing several variants of the basic algorithm. We demonstrate the effectiveness of our approach for models with millions of parameters, including wide residual networks,  for which we show that we can compress CIFAR-10 by a factor of 2 and SVHN by a factor of 3 with less than $0.05\%$ loss of test accuracy. Furthermore, we offer several extensions to our framework, including constructing joint coresets for multiple models and dictionary selection for compressed sensing. The batch active learning application presented in this work is based on \citet{borsos2020semisupervised} with significant performance improvements.

\section{Background and Related Work}

\looseness -1 In this section, we present relevant works on coresets and other data summarization techniques. The list of presented approaches is by no means exhaustive---we refer to \citet{10.5555/2031416, phillips2016coresets, bachem2017practical, feldman2020introduction} for surveys about the topic.

Coresets are commonly defined as weighted subsets of the data. The types of theoretical guarantees provided by coresets, however, vary significantly. Consequently, a wide range of coresets construction algorithms have been proposed, the vast majority of which are applicable to a specific model only. The most common type of guarantees for coresets are uniform multiplicative approximation guarantees: for a given family of nonnegative real functions $\mathcal{F}$ on the space $\mathcal{X} \subseteq \mathbb{R}^d$, we want to find the coreset $C$ as the subset of the data $X=\{x_i\}_{i=1}^n$ and the associated weights $w$ such that for  $\delta, \, \epsilon \in (0,1)$,  $\left| \sum_{x\in C} f(x)  w(x)  -  \sum_{x\in X} f(x) \right| \leq \epsilon  \sum_{x\in X} f(x)$ holds with probability $1-\delta$  \emph{uniformly} for all $ f \in \mathcal{F}$---we refer to coresets with such guarantees as \ecs.

Earliest approaches for constructing {\ecs} appear in computational geometry \citep{agarwal2005geometric} and rely on exponential grids \citep{har2004coresets}.  \citet{feldman2011unified} provide a unified \ec-construction framework based importance sampling via the data points' sensitivity \citep{langberg2010universal}. While efficient and effective for several models such as $k$-median \citep{feldman2011unified}, logistic regression \citep{huggins2016coresets}, and Gaussian mixture models \citep{lucic2017training}, this framework relies bounding the sensitivity, which is a nontrivial, model-specific task. Moreover, the required coreset size  depends also on the pseudo-dimension of the function class $\mathcal{F}$, limiting the framework to low-complexity models (compared to the number of points $n$).

\looseness -1 Closely related to our work are coresets that require guarantees only with respect to the \emph{optimal solution} on coreset instead of uniform guarantees. The majority of constructions providing such guarantees are deterministic (in contrast to the sampling-based  sensitivity-framework): \citet{DBLP:conf/soda/BadoiuC03} provide a greedy forward selection of coresets for the minimum enclosing ball problem of size independent of both $n$ and $d$; based on the latter, \citet{JMLR:v6:tsang05a} propose the Core Vector Machine, which selects a superset of support vectors for SVM in linear time, in a forward greedy manner. \citet{DBLP:journals/talg/Clarkson10} points out that these approaches are instances of the Frank-Wolfe algorithm \citep{Frank1956}. When evaluating the optimal solution found on the coreset (in the unweighted case) on the full data, \citet{pmlr-v37-wei15} show that the resulting set function is submodular in the case of nearest neighbors and naive Bayes, hence greedy selection based on marginal gains guarantees a $1-1/e$ approximation factor to the cost of the solution on the best coreset of a given size.
Similarly to the case of \ecs, existing deterministic coreset construction algorithms are also highly model-specific, and no algorithm has been shown to succeed over a wide range of models.

\looseness -1 Other notable approaches to data summarization include Hilbert coresets \citep{campbell2019automated}, that formulates coreset selection as a sparse vector sum approximation in a Hilbert space with a chosen inner product. The key challenge is choosing an appropriate Hilbert space such that the resulting coreset is a good summary for the original problem. Instead of selecting subsets of the data, data set distillation \citep{wang2018data, lorraine2019optimizing, zhao2021dataset} synthesizes data points by optimizing the features to obtain the summary. Whereas data set distillation creates small summaries effectively, the optimization process is computationally intensive due to the large number of parameters (number of pixels). Consequently, widely used image classification data sets (e.g., CIFAR-10, SVHN, ImageNet) cannot be compressed using existing data distillation techniques without a significant loss in test accuracy compared to training on the full data.

Subset selection has also been formulated as a bilevel optimization problem. This is implicit in the work of \citet{pmlr-v37-wei15}, where the optimization problem can be collapsed into a single-level problem due to closed-form solutions.  Explicit bilevel formulations have been explored in the context of dictionary selection \citep{icml2010_073}. Furthermore, \citet{tapia2019} analyze sensor subset selection as a bilevel optimization problem.  While they use a similar strategy to the one developed here, we investigate considerably different settings of weighted data summarization for a wide range of models and applications.

Techniques related to coresets have also been explored in stochastic optimization. The goal in these works is to improve convergence by selecting better summaries for minibatches than uniform sampling. The challenge here is to design an effective but lightweight selection strategy with negligible computational cost compared to the optimization. With these considerations, \citet{mirzasoleiman2020coresets} propose greedy subset selection based on the submodular facility location problem. Concurrently to our work, \citet{glister} formulate the selection of points based on a bilevel optimization problem that is the unweighted equivalent of our proposed method (with the outer objective defined on the validation). We note that their solution based on Taylor expansion is equivalent to our proposed greedy forward selection with Hessians approximated by the identity matrix in the implicit gradient (Equation  \eqref{eq:implicit-gradient}).

\section{Coresets via Bilevel Optimization}

In this section, we propose a \emph{generic coreset construction framework} that does not rely on model-specific derivations while---in contrast to other coreset constructions---being effective also for advanced models such as deep neural networks. In the design of our generic framework, we focus  on approximation guarantees related to the \emph{solution} on the coreset: informally, we define a ``good'' coreset for a model to be a weighted subset of the data, such that when the model is trained on the coreset, it will achieve a low loss on the full data set. This informal statement translates naturally into a \emph{bilevel optimization} problem with cardinality constraints. We propose several approaches for tackling the resulting combinatorial optimization problem in this section, while we evaluate our proposed method empirically in Section \ref{sec:experiments} in a variety of settings.

\subsection{Problem Setup}

Let us consider a supervised setting with data set $\mathcal{D}=\{(x_i,y_i)\}_{i=1}^n$  and let us introduce the nonnegative weight vector $w=[w_1, \dots,w_n] \in  \reals_+^n$ for representing the coreset $C_w$. Here, $(x_i,y_i,w_i) \in  C_w$ iff $w_i > 0$, i.e., points with $0$ weights are not part of the coreset. Let us denote the model by $h$, its parameters by  $\theta$ and let $\ell$ be the loss function. Furthermore, for brevity, let $\ell_i(\theta) :=\ell(h_\theta(x_i), y_i)$.

\looseness -1 We can formalize our coreset requirement stated in the introduction: we want to find a coreset $C_{\hat{w}}$ of size $m$ such that if we solve the weighted empirical risk minimization (ERM)
problem on the coreset, $\theta_{\hat{w}}^* \in \argmin_\theta \sumin \hat{w}_i\ell_i(\theta)$, then $\theta_{\hat{w}}^*$ is a good solution for the ERM problem on the full data set  $\min_\theta  \sumin \ell_i(\theta) $. We can write this optimization problem as
\begin{equation}\label{eq:bilevel-coreset}
            \begin{aligned}
             \hat{w} \in \argmin_{w\in\reals_+^n,\,||w||_0\leq m}  & \sumin \ell_i(\theta^*(w)) \\
             \textrm{s.t.} \quad & \theta^*(w) \in \argmin_\theta \sumin w_i\ell_i(\theta) ,
            \end{aligned}
        \end{equation}
\looseness -1 where the $m$-sparse vector $\hat{w}$ indicates the selected points' indices and weights at nonzero positions. We note that, although we formulated the problem in the supervised setting, the construction can be easily adapted to semi-supervised and unsupervised settings, as we will demonstrate in Section \ref{sec:experiments}.
Problem~\eqref{eq:bilevel-coreset} is an instance of {\em bilevel optimization}, where we minimize an {\em outer} objective, here $\sumin \ell_i(\theta^*(w))$, which in turn depends on the solution $\theta^*(w)$ to an {\em inner} optimization problem $\argmin_\theta \sumin w_i\ell_i(\theta)$. Before presenting our proposed algorithm for solving \eqref{eq:bilevel-coreset}, we discuss some background on bilevel optimization and its importance in machine learning.

\subsection{Background on Bilevel Optimization} \label{subsec:bilevel-coresets-background}

Modeling hierarchical decision-making processes \citep{von1952theory, vicente1994bilevel}, bilevel optimization has witnessed an increasing popularity in machine learning recently. Bilevel optimization has found applications ranging from meta-learning \citep{finn2017model, li2017meta}, to hyperparameter optimization \citep{pedregosa2016hyperparameter, pmlr-v80-franceschi18a} and neural architecture search \citep{liu2018darts}. 

Suppose $g:\Theta \times \Omega  \rightarrow \reals$ and $f:\Theta \times \Omega  \rightarrow \reals$ are continuous functions, then we call

\begin{equation}\label{eq:bilevel-general}
            \begin{aligned}
            \min_{w \in \Omega} \quad  G(w):=&g(\theta^*(w), w)\\
            \textrm{s.t.} \quad & \theta^*(w) \in \argmin_{\theta \in \Theta} f(\theta, 
            w)
            \end{aligned}
\end{equation}
a bilevel optimization problem with the outer (upper level) objective $\min_{w} \! g(\theta^*\!(w), w)$ and the inner (lower level) objective $\theta^*\!(w) \!\in\! \argmin_{\theta} \! f(\theta,  w)$.  

\looseness -1 Bilevel programming, even in the linear case, is generally NP-hard \citep{vicente1994descent}.  Despite the challenge of non-convexity, first-order methods for solving bilevel optimization problems are successful in many applications \citep{finn2017model, pedregosa2016hyperparameter, liu2018darts}.  A common simplifying assumption for achieving asymptotic convergence guarantees is that the inner problem's solution set in Equation \eqref{eq:bilevel-general} is a \emph{singleton} \citep{pmlr-v80-franceschi18a}, fulfilled if $f$ is strongly convex in $\theta$. 

\looseness -1 First-order bilevel optimization solvers can be further categorized based on how they evaluate or approximate the implicit gradient $\nabla_w G(w)$, for which it is necessary to consider the change of the best response $\theta^*$ as a function of $w$. The first class of approaches defines a recurrence relation $\theta_t = \varphi(\theta_{t-1}, w)$: the recurrence is unrolled, truncated to $T$ steps, and $\nabla_w G(w)$ is approximated by differentiation through the unrolled iterations either by forward- or reverse-mode automatic differentiation \citep{franceschi2017forward}. Whereas choosing a small number of unrolling steps $T$ can potentially introduce bias in the gradient estimation \citep{wu2018understanding}, a large $T$ can incur a high computational cost (either in time or space complexity) when $\Theta$ and $\Omega$ are high-dimensional.

The second class of first-order bilevel optimization solvers obtain the gradient implicitly: under the assumption that $f$ is twice differentiable, the constraint $\theta^*(w) \in \argmin_{\theta \in \Theta} f(\theta, w)$ can be relaxed to $\frac{\partial f(\theta, w)}{\partial \theta}\big|_{\theta =\theta^*} = 0$, which is tight when $f$ is strictly convex. 
     A crucial result for obtaining  $\nabla_w G(w)$ is the \emph{implicit function theorem} applied to $\frac{\partial f(\theta, w)}{\partial \theta}\big|_{\theta =\theta^*} = 0$. Combined with the total derivative and the chain rule, we get
    \begin{equation}
            \begin{aligned}
            \frac{\partial  G(w)}{\partial w} = \frac{\partial g}{\partial w} - \frac{\partial g}{\partial \theta}   \left(\frac{\partial^2 f}{\partial \theta \partial \theta^\top}\right)^{-1}  \frac{\partial^2 f}{ \partial \theta \partial w^\top}
            \end{aligned}\label{eq:implicit-gradient},
        \end{equation}
        where the partial derivatives with respect to  $\theta$ are evaluated at $\theta^*(w)$. However, the Hessian inversion in Equation \eqref{eq:implicit-gradient} is computationally intractable for models with a large number of parameters. Efficient inverse Hessian-vector product approximations can be obtained by approximately solving the linear system $\frac{\partial^2 f}{\partial \theta \partial \theta^\top} x = \left(\frac{\partial g}{\partial \theta}\right)^\top$ with conjugate gradients \citep{pedregosa2016hyperparameter} or by approximating the inverse Hessian with the Neumann series  $\left(\frac{\partial^2 f}{\partial \theta \partial \theta^\top}\right)^{-1} = \lim_{T \to \infty} \sum_{i=0}^T \left(I - \frac{\partial^2 f}{\partial \theta \partial \theta^\top}\right)^i\quad$ \citep{lorraine2019optimizing}. These approximations enable the scalability of first-order optimization methods based on implicit gradients to high-dimensional $\Theta$ and $\Omega$.

\looseness -1 In this work, we use the framework of bilevel optimization to generate coresets. We assume that $f$ is twice differentiable and that the inner solution set is a singleton. We use first-order methods based on implicit gradients for solving our proposed bilevel optimization problems due to their flexibility and scalability.

\subsection{Constructing Coresets via Incremental Subset Selection (BiCo)} \label{subsec:bilevel-incremental}

\looseness -1 In the previous section, we presented different approaches for solving bilevel optimization problems. However, an additional challenge in our coreset formulation \eqref{eq:bilevel-coreset} is  the cardinality constraint $||w||_0\leq m$. One approach for this combinatorial problem would be to treat $G$ as a set function  and  increment the set of selected points greedily by inspecting marginal gains.
Unfortunately, for general losses, this approach comes at a high cost: at each step, we must solve the bilevel optimization problem of finding the optimal weights for each candidate point we consider to add to the coreset. This makes greedy selection based on marginal gains impractical for large models.

\looseness -1 We propose an efficient solution summarized in Algorithm \ref{alg:bilevel_coreset} based on {\em cone constrained generalized matching pursuit} \citep{Locatello2017a}.
    Consider the atom set $\mathcal{A}$ corresponding  to the standard basis of $\reals^n$. Similarly to the Frank-Wolfe algorithm, generalized matching pursuit proceeds by incrementally increasing the active set of atoms that represents the solution by selecting the atom that minimizes the linearization of the objective at the current iteration. Growing  the atom set incrementally can be stopped when the desired size $m$ is reached, and thus the $\norm{w}_0 \leq m$ constraint is active.  We note that, by imposing a bound on the weights and optimizing in the convex hull of atoms, this approach would be equivalent to the Frank-Wolfe algorithm, which has already been applied in coreset constructions in settings where $G(w)$ is convex \citep{DBLP:journals/talg/Clarkson10}. The novelty in our work is the bilevel formulation that allows to extend this simple approach to general twice-differentiable models. 
    
    \begin{figure}[t!]
         \centering
      \begin{algorithm}[H]
        \caption{Bilevel Coreset (BiCo)}
       \label{alg:bilevel_coreset}
    \begin{algorithmic}[1]
       \State {\bfseries Input:} Data $\mathcal{D} =\{(x_i,y_i)\}_{i=1}^n$, coreset size $m$
       \State {\bfseries Output:} weights $w$ encoding the coreset
       \State $w = [0, \dots ,0]$
       \State Choose $i\in [n]$ randomly, set $w_i=1$, $S_1 = \{i\}$
       \For{$t \in [2, ..., m]$}
       \State Find $w^*_{S_{t-1}}\in \reals_+^n$ local min of $G(w)$ s.t. $\textup{supp}(w^*_{S_{t-1}}) = S_{t-1}$
        \State $k^* = \argmin_{k\in [n]} \nabla_{w_k} G(w^*_{S_{t-1}})$
        \State $S_t=S_{t-1} \cup \{ k^* \}$,  $w_{k^*} = 1$
       \EndFor
       \State Find $w^*_{S_m}$ local min of $G(w)$ s.t. $\textup{supp}(w^*_{S_m}) = S_m$
    \end{algorithmic}
    \end{algorithm}
\end{figure}%

       Suppose a set of atoms $S_{t-1} \subset \mathcal{A}$ of size $t-1$ has already been selected. Our method proceeds in two steps. First, the bilevel optimization problem \eqref{eq:bilevel-coreset} is restricted to  weights $w$ having support $S_{t-1}$, i.e., $w$ can only have nonzero entries at indices in $S_{t-1}$. Then we optimize to find the weights $w_{S_{t-1}}^*$ with support restricted to $S_{t-1}$ that represent a local minimum of $G(w)$ defined in Eq.~\eqref{eq:bilevel-general} with $g(\theta^*(w),w)=\sumin \ell_i(\theta^*(w))$ and $f(\theta,w)=\sumin w_i\ell_i(\theta)$---i.e., we use the algorithm's corrective variant, where, once a new atom is added, the weights are reoptimized by gradient descent using the implicit gradient with projection to nonnegative weights (line 6 in Algorithm \ref{alg:bilevel_coreset}). Once these weights are found, the algorithm increments $S_{t-1}$ with the atom that minimizes the first-order Taylor expansion of the outer objective around  $ w_{S_{t-1}}^*$,
    \begin{equation}
            k^* = \argmin_{k\in [n]} e_k ^\top \nabla_w G(w_{S_{t-1}}^*),
        \label{eq:selection-heuristic}
    \end{equation}
    where $e_k$ denotes the $k$-th standard basis vector of $\reals^n$.
    In other words, the chosen point is the one with the largest negative implicit gradient.

    \looseness -1 We can gain insight into the selection rule in Equation  \eqref{eq:selection-heuristic} by expanding $\nabla_w G$ using Equation  \eqref{eq:implicit-gradient}. For this,  we use the inner objective $f(\theta, w) = \sumin w_i  \ell_i(\theta)$ without regularization for simplicity. Noting that $\frac{\partial^2 \sumin w_i  \ell_i(\theta)}{ \partial w_k \partial \theta^\top} = \nabla_\theta \ell_k(\theta)$, we can expand Equation  \eqref{eq:selection-heuristic} to get
    \begin{equation}
    k^*= \argmax_{k \in [n]} \nabla_\theta \ell_k(\theta^*)^\top \Bigg(\frac{\partial^2 \sumin w^*_{S_{t-1},i}  \ell_i(\theta^*) }{\partial \theta \partial \theta^\top} \Bigg)^{-1}  \nabla_\theta \sumin \ell_i(\theta^*).
        \label{eq:explicit-selection-rule}
    \end{equation}
    Thus, with the choice $g(\theta)=\sumin \ell_i(\theta)$, the selected point's gradient has the largest bilinear similarity with $\nabla_\theta \sumin \ell_i(\theta) $, where the similarity is parameterized by the inverse Hessian of the inner problem. 
    
\subsection{Connections and Guarantees} \label{sec:connections-guarantees}
This section presents the connections of our approach to influence functions, experimental design and discusses its theoretical guarantees.
\subsubsection{Connection to Influence Functions} 
Our approach is closely related to incremental subset selection via influence functions. The \emph{empirical influence function}, known from robust statistics \citep{cook1980characterizations}, denotes the effect of a single sample on the estimator. Influence functions have been used by \citet{koh2017understanding}  to understand the dependence of  neural network predictions on a single training point and to generate adversarial training examples. To uncover the relationship between our method and influence functions, consider the influence of the $k$-th point on the outer objective. Suppose we have already selected the subset $S$ and found the corresponding weights $w_S^*$. Then, the influence of point $k$ on the outer objective is
\begin{equation*}
    \begin{aligned}
    \quad & \mathcal{I}(k)   := - \frac {\partial \sumin  \ell_i(\theta^*) }{\partial \eps} \bigg|_{\eps =0}, \quad
    \textrm{s.t.}  & \theta^* = \argmin_\theta \sumin  w^*_{S,i}\,  \ell_i(\theta) +  \eps\ell_k(\theta). \\
    \end{aligned}\label{eq:influence-func}
\end{equation*}

\begin{proposition}\label{prop:infl-fns} Under twice differentiability and strict convexity of the inner loss, the choice $\argmax_k \mathcal{I}(k) $ corresponds to the selection rule in Equation \eqref{eq:explicit-selection-rule}.
\end{proposition}

According to Proposition \ref{prop:infl-fns}, Algorithm \ref{alg:bilevel_coreset} can be interpreted as incremental identification of influential points. However, this algorithm is computationally prohibitive for practically relevant models---we address this issue in  
Section \ref{subsec:variants}.

\subsubsection{Connection to Experimental Design} 
Let us instantiate our approach for the problem of weighted and regularized least squares regression. In this case, the inner optimization problem  
    $\hat{\theta}(w)=\argmin_\theta \sumin w_i ( x_i^\top \theta - y_i)^2+\lambda\sigma^2||\theta||_2^2$, where weights are assumed to be binary, and non-zero for a finite subset $S$,  admits a closed-form solution
    \begin{equation}\label{eq:closed-form-main}
    \hat{\theta}_S = (X_S^\top X_S + \lambda \sigma^2 I)^{-1}X_S^\top y_S.
\end{equation}

    For this special case, we can identify connections to the literature on
{\em optimal experimental design} \citep{Chaloner1995}.
In particular, the data summarization with the outer objective defined as $g(\hat{\theta}) = \frac{1}{2n}\Expect_{\epsilon,\theta}\left[\sum_{i=1}^{n} (x_i^\top \hat{\theta} - y_i)^2\right]$  is closely related to {\em Bayesian $V$-optimal design}, as the following proposition shows. 

\begin{proposition}\label{prop:cvg-main} Under the Bayesian linear regression assumptions $y = X \theta + \epsilon$, $\epsilon \sim \mathcal{N}(0,\sigma^2I)$ and $\theta \sim \mathcal{N}(0,\lambda^{-1})$, let 
    $g_V(\hat{\theta}) = \frac{1}{2n} \Expect_{\epsilon,\theta}\left[\norm{X\theta - X\hat{\theta}}^2_2\right] $ be 
    the Bayesian V-experimental design outer objective. 
For all $\hat{\theta}_S$ in Eq. \eqref{eq:closed-form-main}, we have  \[ \lim_{n \rightarrow \infty}  g(\hat{\theta}_S)- g_V(\hat{\theta}_S)=\frac{\sigma^2}{2}.\]
\end{proposition}
Consequently, it can be argued that, in the large data limit, the optimal coreset with binary weights corresponds to the solution of Bayesian V-experimental design. Further discussion can be found in Appendix \ref{sec:app-exp-design}. By using $g_V$ as our outer objective, solving the inner objective in closed form, we identify the Bayesian V-experimental design objective, 
\begin{equation*}
G(w) = \frac{1}{2n} \trace\left(  X\left(\frac{1}{\sigma^2}X^\top D(w) X + \lambda I\right)^{-1}X^\top \right),    
\end{equation*}
where $D(w):=\textup{diag}(w)$. In Lemma~\ref{lemma:convex} in Appendix \ref{sec:app-exp-design} we show that $G(w)$ is smooth and convex in $w$ when the integrality is relaxed. In such cases, our framework provides additive approximation guarantees, as we show next.

\subsubsection{Theoretical Guarantees}
    If $G(w)$ is smooth and convex, one can show that Algorithm \ref{alg:bilevel_coreset}, being an instance of cone-constrained generalized matching pursuit \citep{Locatello2017a}, provably converges at a rate of $\mathcal{O}(1/t)$.
    
    \begin{theorem}[cf.~Theorem 2 of \citet{Locatello2017a}] Let $G$ be $L$-smooth and convex. After $t$ iterations in Algorithm \ref{alg:bilevel_coreset} we have, 
        $$ G(w_{S_{t}}^*) - G^* \leq  \frac{8L+4\epsilon_1}{t+3}$$
    where $t \leq m$ (number of atoms) and $\epsilon_1 = G(w^*_{S_1}) - G^*$ is the suboptimality gap at $t=1$.
    \end{theorem}
    This result also implies the required coreset size of $m \in \mathcal{O}((L + \epsilon_1)\epsilon^{-1})$, where $\epsilon$ is the desired approximation error.
    We note that our algorithm would be equivalent to the Frank-Wolfe algorithm  by imposing a bound on the weights as commonly done in experimental design \citep{Fedorov1972}.
    Even though, in general, the function $G$ is not convex for more complex models, we  nevertheless demonstrate  the effectiveness of our method empirically in such settings in Section \ref{sec:experiments}.
    
\subsection{Variants} \label{sec:bilevel-coreset-variants}

 \looseness -1 In the previous sections, we presented and analyzed Algorithm \ref{alg:bilevel_coreset}, our basic algorithm for coreset selection. There are three potential bottlenecks when applying this algorithm in practice. Firstly, inverting the Hessian directly in Equation~\eqref{eq:explicit-selection-rule} is infeasible for models with many parameters, thus we must resort to approximating inverse Hessian-vector products with a series of Hessian vector products using the conjugate gradient algorithm or Neumann series (Section \ref{subsec:bilevel-coresets-background}). Secondly, adding points one by one might be too costly to generate larger coresets. Lastly,  running multiple weight update steps might be prohibitive since the algorithm's complexity depends linearly on the number of outer iterations. In this section, we address these issues by presenting several variants of our proposed method.
 
\looseness -1 \subsubsection{Selection via Proxy} 
A natural idea to improve the scalability of the algorithm is to perform the coreset selection on a \emph{proxy}  instead of the original model. This approach is practical for generating small coresets, as there is a trade-off between the complexity of the proxy model and solving the bilevel optimization problem efficiently---the discrepancy between a simple proxy and the original model can result in suboptimal large coreset selection via the proxy. 

We now study the setting when the proxy hypothesis class is a reproducing kernel Hilbert spaces (RKHS). This proxy class is relevant for a wide range of models, including neural networks due to the connection between the Neural Tangent Kernel \citep{jacot2018neural} and the training of infinitely wide neural networks with batch gradient descent.

Let $\kappa(\cdot, \cdot)$ be a positive definite kernel function and let $\HH$ denote the endowed reproducing kernel Hilbert space (RKHS), such that $\HH$ is the completion
of $\textup{span}\{\kappa(x,\cdot), x\in \mathcal{X} \}$
and $\langle h, \kappa(x, \cdot) \rangle_\HH = h(x)$ 
for all $h\in \HH$ and $x \in \mathcal{X}$. Furthermore, let $K$ denote the Gram matrix associated with the data. For solving the bilevel optimization we approximate $\HH$ with a smaller function space $\HH_q$ using the Nyström method \citep{nystrom}.

 \looseness -1 The Nyström method provides a low-rank approximation $\hat{K}$ to the Gram matrix $K$ by selecting a data-dependent basis as a subset of the training data $Q\subseteq [n]$, $|Q|=q$ such that $\hat{K}=K_{[n], Q} K_{Q, Q}^{+} K_{Q, [n]}$. Given the eigendecomposition of $K_{Q, Q}=U D U^\top$, the $q$-dimensional Nyström feature map is given by
    $z(\cdot) = D^{-1/2}U^\top [\kappa(\cdot, x_i), \, i\in Q]$
, such that $\hat{K}_{i,j} =z(x_i)^\top z(x_j) $. The problem of selecting the subset $Q$ has attracted significant interest, where the prominent tool is nonuniform sampling based on leverage scores and its variants \citep{mahoney2009cur}. We use the simplest and computationally most efficient method of uniform sampling for selecting $Q$. With the Nyström approximation, Equation~\eqref{eq:bilevel-coreset} can be rewritten as
\begin{equation}\label{eq:bilevel-coreset-nystrom}
            \begin{aligned}
            \min_{w\in \mathbb{R}_+^n,\,\normp{w}{0}\leq m}  \quad & \sumin   \ell(\theta^{*\top}z(x_i), y_i) \\ 
            \textrm{s.t.} \quad & \theta^* = \argmin_{\theta} \sumin w_i  \ell(\theta^{\top}z(x_i), y_i).
            \end{aligned}
        \end{equation}
For common loss functions $\ell$ (such as cross-entropy or $L_2$) the inner optimization problem is convex, allowing for fast solvers. Thus, in practice, the computational cost is dominated by the complexity of calculating $\hat{K}$ instead of solving the bilevel optimization problem.

\subsubsection{Forward Selection in Batches, Exchange, Elimination}

For building large coresets, however, the discrepancy between the proxy and original model makes the previous approach impractical. Hence, we consider working with the original model and propose additional approximations resulting in several algorithmic variants.

Firstly, since the algorithm's complexity depends linearly on the number of outer iterations, running multiple outer iterations might not be desirable. Secondly, adding points one by one might be too costly for generating larger coresets. 
We propose a simple solution to mitigate the first issue: we restrict the coreset weights to binary, and we do not optimize them. This way, the number of implicit gradient calculations will be equal to the number of coreset points chosen incrementally. We will show experimentally that, despite this simplification, unweighted coresets are very effective in a wide range of applications.
As for the second issue,  we propose the following approaches as alternatives to the basic one-by-one forward selection in Algorithm \ref{alg:bilevel_coreset} and investigate them experimentally in Section \ref{sec:experiments}:

\begin{itemize}
    \item \emph{forward selection in batches}: start with a small random subset, increase the chosen subset by  a batch of $b$ points  having the smallest implicit gradient in each step.
     \item \emph{exchange in batches}: start with a random subset of the desired coreset size, remove $b$ of the chosen points having the largest implicit gradient, and add $b$ new points having the smallest implicit gradient in each step.
    \item \emph{elimination in batches}: start with the full data set, remove $b$ of the chosen points having the largest implicit gradient in each step.
\end{itemize}
Similar ideas are prevalent in experimental design, e.g., the ``excursion'' version of Fedorov's exchange algorithm \citep{Fedorov1972, 10.2307/1267940}. The significant difference to our approach is that, for the  experimental design objectives available in closed-form, the selection algorithm can easily evaluate the exact effect of adding or removing points---in our case, this is prohibitively expensive, thus we must resort to our proposed heuristic based on first-order Taylor expansions. 
Furthermore, we note that the we perform the selection of batches by choosing greedily the $b$ points that have the smallest (largest, in case of elimination) implicit gradient. Exploring approaches that enforce diversity of the points in the selected batch is a promising future direction.

\looseness -1 The correspondence between our selection rule and  influence functions (Section \ref{sec:connections-guarantees}) brings into relevance the work of \citet{wang2018data}, who propose to curate the data set by removing unfavorable training points to increase the generalization performance of the model. Their method is essentially a two-stage backward elimination algorithm based on influence functions. Hence, our elimination method is a multi-stage extension of \citet{wang2018data}.

\subsubsection{Bilevel Coresets via Regularization}

Another approach to solving the cardinality constrained bilevel optimization for coreset selection (Equation \eqref{eq:bilevel-coreset}) is to transform the $\normp{w}{0}$ constraint into a sparsity-inducing regularizer, e.g., into an $L_1$-penalty, in the spirit of Lasso \citep{tibshirani1996regression}. However,  this approach fails for Equation \eqref{eq:bilevel-coreset}, since the solution of the inner optimization problem is a minimizer also when the weights are rescaled by a common factor. 
One attempt for mitigating this issue is to require the inner optimization problem to be regularized, and add an $L_1$-penalty to the outer objective:
\begin{equation} \label{eq:bilevel-l1-reg}
    \begin{aligned}
    \min_{w\in \mathbb{R}_+^n} \quad &  \sumin   \ell_i(\theta^*) + \beta \normp{w}{1} \\ \textrm{s.t.} \quad &  \theta^* = \argmin_\theta \sumin w_i  \ell_i(\theta) +\lambda \normp{\theta}{2}^2.
    \end{aligned}
\end{equation}
\looseness -1 Now, for fixed $\lambda$, rescaling the weights by an arbitrary common factor in Equation~\eqref{eq:bilevel-l1-reg} does affect the inner optimization problem's solution. However, there exists a wide range of regularizers $\lambda$ in practice (for some problems orders of magnitude away, e.g., $\lambda \in [10^{-6}, 10^{-4}])$ that, for the same fixed weights, have an almost identical outer loss. This introduces an issue: for a fixed $\lambda$, coreset weights of different magnitudes will produce good solutions to the inner problem, whereas the outer loss is highly susceptible to the scale of the weights. Thus, tuning $\lambda$ and $\beta$ jointly to achieve a desired coreset size in this setting is cumbersome in practice.

\begin{figure*}[t!]
\centering
  \includegraphics[width=0.6\linewidth]{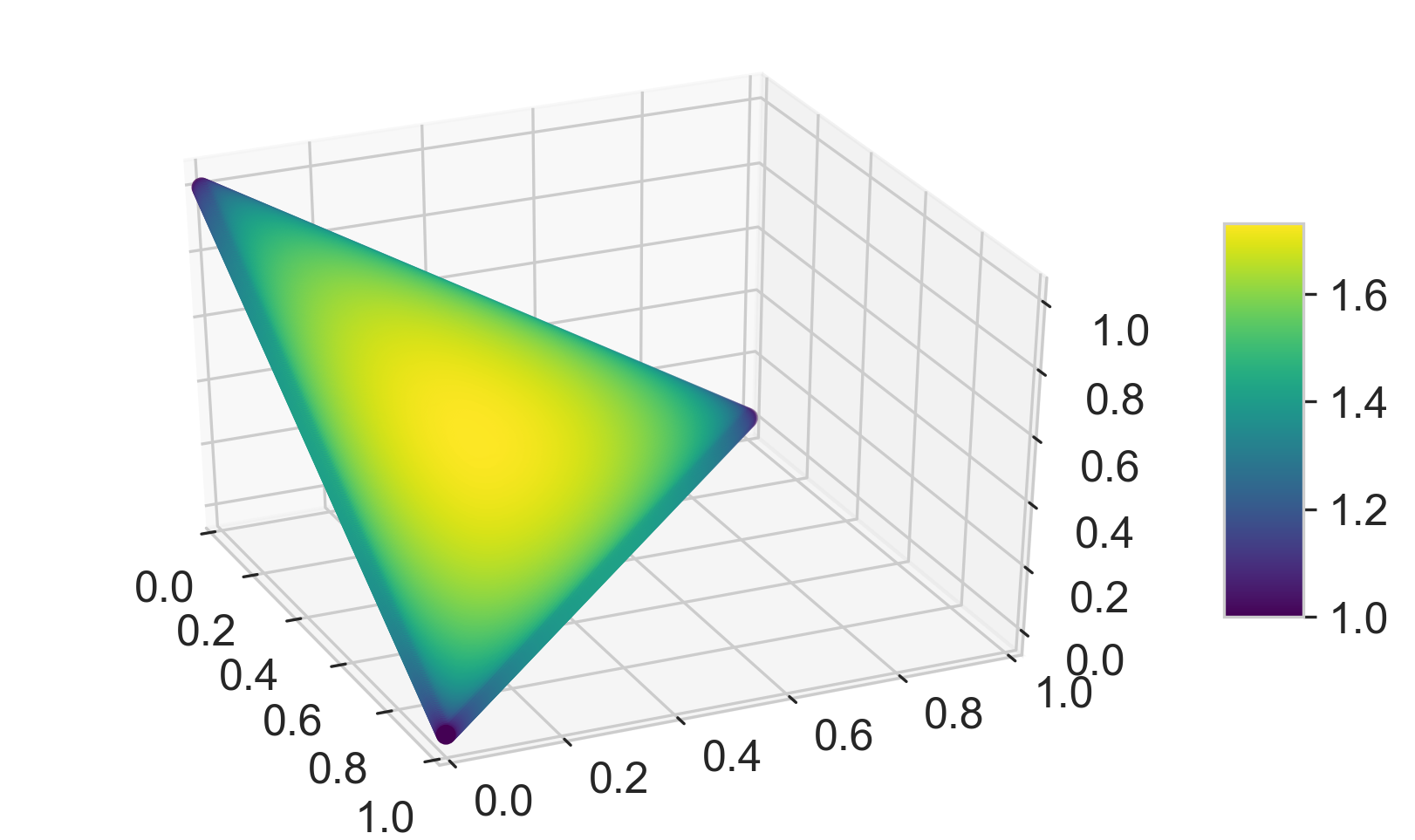}
  \caption{$L_{1/2}$ penalty in three dimensions restricted to the simplex. The value of the penalty decreases towards the edges of the simplex, inducing sparsity.\label{fig:lq-penalty}}
  \vspace{-4mm}
\end{figure*}

Therefore, we propose another regularized version of the problem that is easier to tune, it does not require a regularized inner problem, and it produces solutions with better empirical performance. First, we restrict the weights to the $n$-dimensional simplex $\Delta_n$, such that $\sumin w_i = 1$. Now, since $\normp{w}{1}=1$, we should use another sparsity-inducing penalty in the outer loss: any $L_q = \sumin w_i^{q}$ with $q \in (0, 1)$, where we choose $q=1/2$ in this work---Figure \ref{fig:lq-penalty} shows $L_{1/2}$ in three dimensions restricted to the simplex. Thus, our proposed regularized bilevel coreset selection problem (inner regularization is also supported) is 
\begin{equation} \label{eq:bilevel-lq-reg}
    \begin{aligned}
    \min_{w\in \Delta_n} \quad &  \sumin   \ell_i(\theta^*) + \beta \sumin \sqrt{w_i} \\ \textrm{s.t.} \quad &  \theta^* = \argmin_\theta \sumin w_i  \ell_i(\theta).
    \end{aligned}
\end{equation}

For optimizing the bilevel problem in Equation~\eqref{eq:bilevel-lq-reg}, we apply first-order methods using the implicit gradient, which, based on the total derivative, the chain rule, and the implicit function theorem is
\begin{equation} \label{eq:bilevel-reg-implicit-grad}
   \beta \frac{\partial\sumin \sqrt{w_i}}{\partial w}  - \frac{\partial \sumin  \ell_i(\theta^*)}{\partial \theta}  \left(\frac{\partial^2 \sumin w_i  \ell_i(\theta^*)}{\partial \theta \partial \theta^\top} \right)^{-1} \frac{\partial^2 \sumin w_i  \ell_i(\theta^*)}{\partial \theta \partial w^\top}.
\end{equation}
Additionally, since the weights are constrained to $\Delta_n$, we project the weights after each gradient descent step to $\Delta_n$ using the efficient Euclidean projection step proposed by \citet{duchi2008efficient}. Furthermore, to ensure finite derivatives in Equation \eqref{eq:bilevel-reg-implicit-grad} due to the $L_{1/2}$-penalty, we mix the projected weight vector with the identity vector to avoid exactly $0$ weights. Our proposed method summarized in Algorithm \ref{alg:bilevel-coreset-regularized}.
    \begin{figure}[t!]
         \centering
      \begin{algorithm}[H]
        \caption{Bilevel Coreset via Regularization}
       \label{alg:bilevel-coreset-regularized}
    \begin{algorithmic}[1]
       \State {\bfseries Input:} Data $\mathcal{D} =\{(x_i,y_i)\}_{i=1}^n$, $T$, regularizers $\lambda$, $\beta$
       \State {\bfseries Output:} weights $w$ encoding the coreset
       \State $w = [1/n, \dots, 1/n]$, $\epsilon = 10^{-8}$
       \For{$it \in [1, \dots, T]$}
       \State Find $\theta^* = \argmin_\theta \sumin w_i  \ell_i(\theta) +\lambda \normp{\theta}{2}^2$
        \State Update $w$ by gradient descent using Equation \eqref{eq:bilevel-reg-implicit-grad}
        \State $\tilde{w} = \argmin_{w' \in \Delta_n} \normp{w' - w}{2}$ \Comment{\citet{duchi2008efficient}}
        \State $w = (1-\epsilon)\tilde{w} + \epsilon \mathbb{1}_n$
       \EndFor
        \State $w[w < 10^{-4}] = 0$

    \end{algorithmic}
    \end{algorithm}
    \vspace{-4mm}
\end{figure}

\looseness -1 In practice, to reach a desired coreset size $m$, we tune our hyperparameters $\lambda$ and $\beta$ as follows. We first tune $\lambda$ based on the validation performance by solving the inner optimization problem with $w = [1/n, \dots, 1/n]$; after the tuning, $\lambda$ will be fixed. We set $\beta$ to small value, e.g., $\beta=10^{-7}$, start the loop in Algorithm \ref{alg:bilevel-coreset-regularized}, and we monitor the number of selected coreset points: if the number of the selected coreset points was plateauing in recent iterations, then we increase the sparsity penalty by doubling $\beta$---we use the doubling until the desired coreset size $m$ is reached.

\looseness -1 This approach is most practical for generating large weighted coresets. For generating small coresets, however, it has the disadvantage of increased computational cost compared to greedy forward selection, as the first optimization step already involves fitting the model on the full data.

\section{Extensions and Applications of Bilevel Coresets} \label{sec:extensions}

\looseness -1 Our framework has the advantage of flexibility in handling extensions for the outer and inner objectives (Equation \eqref{eq:bilevel-coreset}), which makes it   applicable in a wide range of settings. In this section, we present the framework's extension to constructing joint coresets for multiple models and its applications in continual learning, streaming, batch active learning, and dictionary selection in compressed sensing.

\subsection{Joint Coresets} \label{subsec:joint-coresets}
 One application of our framework is speeding up model selection and hyperparameter tuning by performing these on the coreset instead of the full data. For this, we expect the coreset to be transferable to multiple models, whereas our formulation (Equation \eqref{eq:bilevel-coreset}) is tied to a model and a loss function. A simple idea to construct a coreset with better transferability is to ensure that it is a suitable coreset for multiple models. This is straightforward to achieve within our framework---for brevity, we present the idea for two models: consider models $f$ and $g$, and denote their parameters by $\theta_f$ and $\theta_g$. The problem of generating the joint coreset can be formulated as

\begin{equation}\label{eq:bilevel-multiple}
            \begin{aligned}
              \min_{w\in\reals_+^n,\,||w||_0\leq m}  & \sumin \left( \ell(f_{\theta_f^*}(x_i), y_i) + \lambda \ell(g_{\theta_g^*}(x_i), y_i) \right) \\
             \textrm{s.t.} \quad & (\theta_f^*, \theta_g^*) \in \argmin_{(\theta_f, \theta_g)}\! \sumin\! w_i \left( \ell(f_{\theta_f}\!(x_i), y_i)   + \lambda\ell(g_{\theta_g}\!(x_i), y_i)  \right).
            \end{aligned}
        \end{equation}
\looseness -1 For solving this bilevel problem, we can rely on the previously presented techniques. In practice, if the loss magnitudes are of the same order, we can set $\lambda=1$; an additional heuristic for solving the problem with (batch) forward selection is to perform the selection step alternatingly for each model. We verify the validity of this approach in the next section, where we demonstrate the improvement in the transferability of the coreset to deep convolutional networks.

\subsection{Continual Learning} In contrast to the standard supervised setting, where the learning algorithm has access to an i.i.d. data set $\mathcal{D}=\{(x_i,y_i)\}_{i=1}^n$, continual learning assumes that $\mathcal{D}$ is the union of $T$ disjoint subsets $\mathcal{D}_1, \dots, \mathcal{D}_T$ such that each $\mathcal{D}_i$ contains data drawn from a different i.i.d. distribution. The goal is to learn a model based on the data that arrives sequentially from different tasks, such that the model achieves good performance on all tasks. An additional constraint in the setting is that the model cannot revisit all data from the previous tasks $1,\dots,t-1$ when learning on task $t$. The challenge is to avoid  \emph{catastrophic forgetting} \citep{mccloskey1989catastrophic, french1999catastrophic}, which occurs when the optimization on $\mathcal{D}_t$ degrades  the model's performance significantly on some of $\mathcal{D}_1,\dots,\mathcal{D}_{t-1}$.

\looseness -1 Continual learning with neural networks has received increasing interest recently. The approaches for alleviating catastrophic forgetting fall into three main categories: weight regularization to restrict deviations from parameters learned on old tasks \citep{kirkpatrick2017overcoming, v.2018variational}; architectural adaptations for different tasks \citep{rusu2016progressive}; and replay-based approaches, where samples from old tasks are either reproduced via replay memory \citep{lopez2017gradient} or  generative models \citep{shin2017continual}. 

In this work, we focus on the replay-based approach, which provides strong empirical performance  despite its simplicity \citep{chaudhry2019continual}. In this setting, coresets are natural candidates for the summaries of the tasks to be stored in the replay memory, and we can readily use our coreset construction for the selection. 

For continual learning with replay memory, we employ the following protocol. 
The learning algorithm receives data $\mathcal{D}_1,\dots,\mathcal{D}_T$ arriving in order from $T$ different tasks. At time $t$, the learner receives $\mathcal{D}_t$ but can only access past data through a small number of samples from the replay memory of size $m$. We assume that equal memory is allocated for each task in the buffer, and that the summaries $C_1,\dots,C_T$ are created per task. Thus, the optimization objective at time $t$ is
\[\min_\theta \frac{1}{|\mathcal{D}_t|}\sum_{(x,y) \in \mathcal{D}_t} \ell(h_\theta(x),y) + \beta  
   \sum_{\tau=1}^{t-1} \frac{1}{|\mathcal{C}_\tau|}\sum_{(x,y) \in \mathcal{C}_\tau} \ell(h_\theta(x),y),\]
where $ \sum_{\tau=1}^{t-1} |C_\tau|\!=\!m$ and $\beta$ is a hyperparameter controlling the regularization strength of the loss on the samples from the replay memory. After performing the optimization, $\mathcal{D}_t$ is summarized into  $\mathcal{C}_t$ and added to the buffer, while previous summaries $\mathcal{C}_1,\dots, \mathcal{C}_{t-1}$ are shrunk such that $|\mathcal{C}_\tau| = \lfloor m/t \rfloor$. The shrinkage is performed by running the summarization algorithms on each $\mathcal{C}_1,\dots, \mathcal{C}_{t-1}$ again, which for greedy strategies is equivalent to retaining the first $\lfloor m/t \rfloor$ samples from each summary.

\subsection{Streaming} \label{subsec:streaming}
 We can also apply our coreset construction in the more challenging setting of streaming. In contrast to continual learning, the streaming setting does not define tasks and does not assume i.i.d. data in any portion of the stream. Concretely, in this work, we assume that the learner observes small data batches $\mathcal{D}_1,...,\mathcal{D}_T$ arriving in order, where no i.i.d.~and task boundary assumptions are made. 

As in the case of continual learning, one approach for alleviating catastrophic forgetting in the streaming setting is the retraining on data from the memory replay buffer. 
Denoting by $\mathcal{M}_{t}$ the replay memory at time $t$, the optimization objective at time $t$ for  learning under streaming with replay memory is
\[    \min_\theta  \frac{1}{|\mathcal{D}_t|} \sum_{(x,y) \in \mathcal{D}_t} \ell(h_\theta(x),y) + \frac{\beta}{|\mathcal{M}_{t-1}|}  \sum_{(x,y) \in \mathcal{M}_{t-1}}   \ell(h_\theta(x),y).\]

 Maintaining a coreset of constant size over datastreams is a cornerstone of training nonconvex models in a streaming setting.
 We offer a principled way to achieve this, naturally supported by our framework, using the following idea: two coresets  can be summarized into a single one by applying our bilevel  construction with the outer objective as the loss on the union of the two coresets.

Based on this idea, we use a variant of the merge-reduce framework of \citet{chazelle1996linear}. 
For this, we assume that we can store at most $m$ coreset points in a buffer; we split the buffer into $s$ slots of equal size $m_s:= m/s$. 
We associate values $\beta_i$ with each of the slots, which will be \emph{proportional to the number of points} they represent.  A new batch  is compressed into a new slot with associated default $\beta$, and it is appended to the buffer, which now might contain an extra slot. The reduction to size $m$ happens as follows: select the two consecutive slots $i$ and $i+1$ with smallest $i$ for which  $\beta_i=\beta_{i+1}$ or, if this does not exist, choose the last two slots; then join the  content of the slots ({\em merge}) and create the coreset of the merged data ({\em reduce}). 
The new coreset replaces the two original slots with $\beta_i +\beta_{i+1} $ associated with it. The pseudocode of the construction is shown in Algorithm  \ref{alg:streaming-coreset} in Appendix \ref{sec:app-cl-streaming} together with the illustration of the merge-reduce coreset construction for a buffer with 3 slots and 7 steps in Figure \ref{fig:merge-reduce}. The coreset produced by our construction for a two-layer fully-connected neural network on the imbalanced video stream created from the iCub World 1.0 data set \citep{fanello2013icub} can be seen in Figure \ref{fig:demo-stream}.

\begin{figure*}[t!]
\centering
          \includegraphics[width=\textwidth]{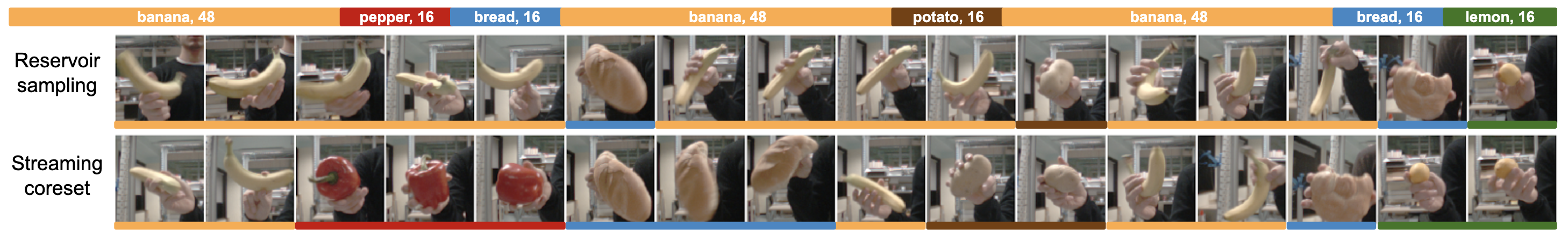}
          \caption{Data summarization on an imbalanced stream of images created from the iCub World 1.0 data set \citep{fanello2013icub}.  Row 1: the stream's composition containing 5 object classes. Row 2: selection by reservoir (uniform) sampling. Row 3: selection by our method. Reservoir sampling misses classes (pepper) due to imbalance and does not choose diverse samples, in contrast to our method.           \label{fig:demo-stream}} 
\end{figure*}

\subsection{Batch Active Learning}

The prominent use cases of our proposed method are scenarios with explicit budget constraints for subset selection. These constraints can be due to  computational resource constraints, as in the case of continual learning and streaming with replay memory, or can relate to the cost of involving human interaction. \emph{Active learning} falls into the latter category and aims to improve the data efficiency of learning algorithms by interleaving training rounds with selective query of the labels for informative unlabeled points from human experts.

\looseness -1 The active learning setting assumes that, while unlabeled data is available abundantly, acquiring labels involves the cost of relying on human expertise. In the \emph{pool-based} setup, each active learning round consists of training the model using the already labeled  data and choosing points from the unlabeled pool for label acquisition. The challenge in this setting is to select the most informative samples, i.e., the samples with the highest potential of reducing the model's generalization error.
When the cost of performing a new training round after every single acquired label is considered, active learning becomes computationally unattractive. \emph{Batch active learning} tackles this issue by acquiring labels for a batch of points in a single round but faces the additional challenge of ensuring diversity between the chosen points. 

Active learning and its batch variant have received significant attention  \citep{mackay1992information, lewis1994sequential, balcan2007margin,  hoi2006batch, guo2008discriminative, NIPS2019_8925}. Although vastly available unlabeled data is assumed in this setting, most active learning approaches ignore the unlabeled pool while training the model in a supervised manner on the labeled pool only. On the other hand, recent advances in semi-supervised learning (SSL) have shown significant performance improvements  of models trained with only a small number of labeled samples. Prominent SSL methods in the image domain include Mean Teacher \citep{NIPS2017_68053af2}, MixMatch \citep{berthelot2019mixmatch} and its improvement, FixMatch \citep{NEURIPS2020_06964dce}. These methods achieve the following CIFAR-10 test accuracies: Mean teacher - $78.5\%$ with $1000$ labeled samples;   MixMatch - $88.2\%$ with $250$ labeled samples; FixMatch - $95\%$ with $250$ labeled samples.

The success of semi-supervised methods suggests that using the unlabeled data pool in active learning for acquisition only is suboptimal. Based on this observation, early approaches propose to combine active learning with SSL for Gaussian fields~\citep{DBLP:conf/icml/ZhuGL03} and SVMs~\citep{1467457, leng2013combining}. In the context of semi-supervised active learning with neural networks, \citet{sener2018active} investigate SSL training and selecting points for label acquisition that represent the $k$-centers of the embeddings in the last layer. \citet{song2019combining} show that training in a semi-supervised manner with MixMatch and selecting the candidates for label query with standard acquisition functions improves the active learner's generalization performance compared to uniform sampling.
\citet{gao2019consistency} propose to train in each acquisition round with MixMatch and query the points that produce the most inconsistent predictions  when undergoing random data augmentations, as measured by the sum of per-class variances in the predicted class probabilities. We compare our proposed acquisition strategy with these methods empirically.

We propose a simple yet highly effective label acquisition strategy based on bilevel coreset construction that works in the semi-supervised batch active learning setup. The basic idea is the following: in each round of active learning, we train the semi-supervised learner and use its predictions to provide labels for the samples in the unlabeled pool; then, using these ``pseudo-labels'', we construct the coreset of the unlabeled pool and query the true labels for the selected points. This strategy naturally accommodates the selection of batches and prohibits redundancy in the selected batch by the design of the objective.

\looseness -1  Let us formalize our approach in a single round of batch active learning. Denote the labeled pool  by $\dtrain=\{ (x_i, y_i)\}_{i=1}^{n_{\textup{labeled}}}$ and the unlabeled pool by $\du = \{ x'_i \}_{i=1}^{n_{\textup{unlabeled}}}$.
Let $h$ denote the model and $\theta^*_{\textup{SSL}}$ denote the parameters that minimize the semi-supervised loss---our strategy is oblivious to the choice of the SSL algorithm, it only assumes that the semi-supervised training outperforms supervised training of the model in terms of the generalization error. Lastly, let $\wdu = \{ (x, h_{\theta^*_{\textup{SSL}}}(x)), \; x \in \du \}$  denote the data set of points from $\du$ together with their soft pseudo-labels provided by the semi-supervised learner. 

The goal of batch active learning is to select and query the labels of the most informative subset of the unlabeled data pool $\mathcal M \subseteq \wdu$ of size $m=|\mathcal M|$ that would result in a maximal reduction of the model's generalization error. We propose to select $\mathcal M$ as follows: $\dtrain \cup \mathcal M$ should be the coreset of $\dtrain \cup \wdu$ for training $h$ in a \emph{supervised} manner.  Formally, 
\begin{equation}\label{eq:active-learning-bilevel}
            \begin{aligned}
            \mathcal M := &\argmin_{\mathcal M \subseteq \wdu,\, |\mathcal M|=m} \quad \sum_{(x,y)\in\mathcal{D}_{\textup{train}} }\!\! \ell(h_{\theta^*}(x), y) +  \sum_{(x, \widehat{y})\in\wdu} \!\! \ell(h_{\theta^*}(x), \widehat{y})\\
            &\textrm{s.t.}  \; \theta^* = \argmin_{\theta} \!\! \sum_{(x,y)\in\mathcal{D}_{\textup{train}} } \!\! \ell(h_{\theta}(x), y) + \sum_{(x,\widehat{y})\in \mathcal M} \!\! \ell(h_{\theta}(x), \widehat{y}),
            \end{aligned}
        \end{equation}
where $\widehat{y}$ denote the pseudo-labels. The motivation for the formulation in Equation~\eqref{eq:active-learning-bilevel} is twofold. Firstly, as the coreset of $\wdu$, $\mathcal M$ will contain the most essential points of the unlabeled data pool for supervised training. In case some of these points have been wrongly pseudo-labeled, we expect that querying the correct labels induces a large model change. In the other case, acquiring hard labels benefits the semi-supervised learner in label propagation. Secondly, the coreset selection in Equation~\eqref{eq:active-learning-bilevel} naturally supports batch selection while avoiding redundancy among the selected points. We provide empirical support for this hypothesis in the experiments.

\subsection{Dictionary Selection for Compressed Sensing} \label{subsec:dict-sel}

\looseness -1  In signal compression, a collection of signals (data points) needs to be summarized by a small set of measurements ensuring high-fidelity reconstruction. In this case, instead of selecting data points, we select low-dimensional projections of the data. Despite this difference, due to the generality of our bilevel framework, we can showcase our framework for selecting measurements from a set of dictionary elements in order to improve the compression performance. This problem closely resembles dictionary learning and can be seen as a special case of it, without the individual sparsity constraints \citep{icml2010_073}. The classical greedy method, which can  obey cardinality constraints on the measurement set and thus control the compression ratio, is computationally very expensive: for each element of the dictionary and at each enlargement, the whole dataset needs to be reconstructed. This increases the computational burden by the size of the dictionary, which can be prohibitively large. 

\looseness -1 Classically, the compression is addressed by transforming the data (signal) to a basis with a known redundancy such as a Fourier transform, and subsequently applying a set of \emph{linear} measurements on the data. These are then recovered by imposing a regularization strategy such as the smallest squared squared norm ($L_2$). Alternatively, \citet{chen_2005} proposed to use absolute norm regularization ($L_1$) referred to as \emph{basis pursuit} or  \emph{compressed sensing}. In fact, compressed sensing can provably recover $s$-sparse signals with much smaller set of linear measurements than  $L_2$ regularization, scaling as $\mathcal{O}(s\log(d))$ \citep{donoho2006compressed}. The  measurement vectors, however, must satisfy specific conditions such as restricted isometry property (RIP) \citep{candes2006stable} for this to be guaranteed. Certifying that a measurement matrix has the RIP is known to be NP-hard \citep{bandeira2013certifying}. Constructing these matrices randomly is easy, however, the procedure generates them only with a certain probability \citep{candes2006stable}. 

Given a representative curated data set sufficiently covering all reasonable signals, a natural question is whether one can design a tailored measurement set that improves the compression ratio beyond the randomly generated RIP matrices or other classical measurements. In fact, since the recovery procedure is formulated as an optimization problem, this design question can be captured in the familiar bilevel form:

\begin{equation}\label{eq:bilevel-dictionary}
\begin{aligned}
             \min_{\norm{w}_0\leq k} & \sup_{i \in [n]} \norm{x_i - \hat{x}_i(w)}_2^2  \\
            \quad \textrm{s.t.} \quad& \hat{x}_i = \arg\min_{y_i}  \sum_{j=1}^{m} w_j(a_{j}(x_i - y_i))^2 + \lambda R(y) \quad \forall i \in \mathcal{D},
\end{aligned}
\end{equation}
\looseness -1 where $n$ is the size of representative data set, $a_j \in \mathbb{R}^{d}$ is one of the $m$ elements of the dictionary we select from, and $R(y)$ is the regularization term corresponding to $R(y) = \norm{y}_2^2$ or $R(y) = \norm{y}_1$. The values of $w_i$ are restricted to be binary in this application. While the optimization problems might not always be differentiable, the objective can be reformulated using differentiable relaxation techniques such as the \emph{log-sum-exp} trick to approximate the supremum, and $L_1$ admits a variational reformulation via the so called $\eta$-trick to become differentiable \citep{bach2012etatrick}. In practice, however, using an element of the sub-differential proves to be a  viable strategy. We demonstrate the versatility of our framework by applying it to solve Equation \eqref{eq:bilevel-dictionary} in Section \ref{subsec:dict-sel-exp}.

\looseness -1 Recently, \cite{bora2017compressed} and \cite{jalal2021instance} have demonstrated that the sparsity-inducing regularizers can be substituted by the constraint that the data belongs to the support of a generative model $G(z)$, where $z \in \mathbb{R}^p$ is referred to as the latent space. In this case, the regularization term becomes an indicator function $R(y) = \mathbf{1}_{\exists z ~ | ~ y = G(z)}$, which can be reformulated to get a simplified inner problem $\hat{x_i} = G(z_i)$ and $z_i = \arg\min_z \sum_{j=1}^m w_j(a_j(x_i - G(z)))^2$. The measurements are assumed to be linear as in classical compressed sensing, and the recovery guarantees satisfy similar conditions on measurement vectors as with sparse signals in previous works \citep{jalal2021instance}. Naturally, a more informed measurement selection of $A$ as above can even further reduce the compression ratio.

\section{Experiments} \label{sec:experiments}

\looseness -1 In this section, we  demonstrate the flexibility and effectiveness of our framework for a wide range of models and various settings. We start by evaluating the variants of Algorithm \ref{alg:bilevel_coreset} in Section \ref{subsec:variants},  and we compare our method to model-specific coreset constructions and other data summarization strategies in Section \ref{subsec:comparison}.  We then study our approach in the memory-constrained settings of continual learning and streaming in Section \ref{subsec:cl-streaming}, of dictionary selection in Section \ref{subsec:dict-sel-exp},  and the human-resource constrained setting of batch active learning in Section \ref{subsec:bal}.

\subsection{Variants }\label{subsec:variants}

The basic algorithm (Algorithm \ref{alg:bilevel_coreset}) for bilevel coresets is impractical except for simple models due to its computational cost---we discuss the runtime complexity in Section \ref{subsec:comp-cost}. Hence, we focus on the variants proposed in Section \ref{sec:bilevel-coreset-variants}.  Our target model is multiclass logistic regression, where the feature space is the $q=2048$-dimensional Nyström feature space  of the Convoluational Neural Tangent Kernel (CNTK) proposed by \citet{arora2019} with six layers and global average pooling on CIFAR-10.
In this case, $\theta \in \reals^{q \times c}$,  and $\ell(\theta^\top z(x),y) = -\sum_{j=1}^c y_j \log \frac{\exp(\theta_{\cdot,j}^\top z(x))}{\sum_{j'=1}^c \exp(\theta_{\cdot,j'}^\top z(x))}$, where $y$ is the one-hot encoded label vector, $\ell$ is the cross-entropy loss with softmax, and $z(\cdot)$ is the Nyström feature mapping. In each implicit gradient step, we solve the inner optimization problem iteratively up to a tolerance, and approximate the implicit gradient \citep{pedregosa2016hyperparameter} with $100$ conjugate gradient steps. We split CIFAR-10 into a train and validation set, where the validation set is a randomly chosen $10\%$ of the original training set. We instantiate the outer loss as the sum of training and validation losses, whereas the inner optimization problem is defined on the training set. Further details about the experimental setup can be found in Appendix \ref{sec:app-exp-setup}.

\begin{figure*}[t!]
\centering
  \includegraphics[width=0.9\linewidth]{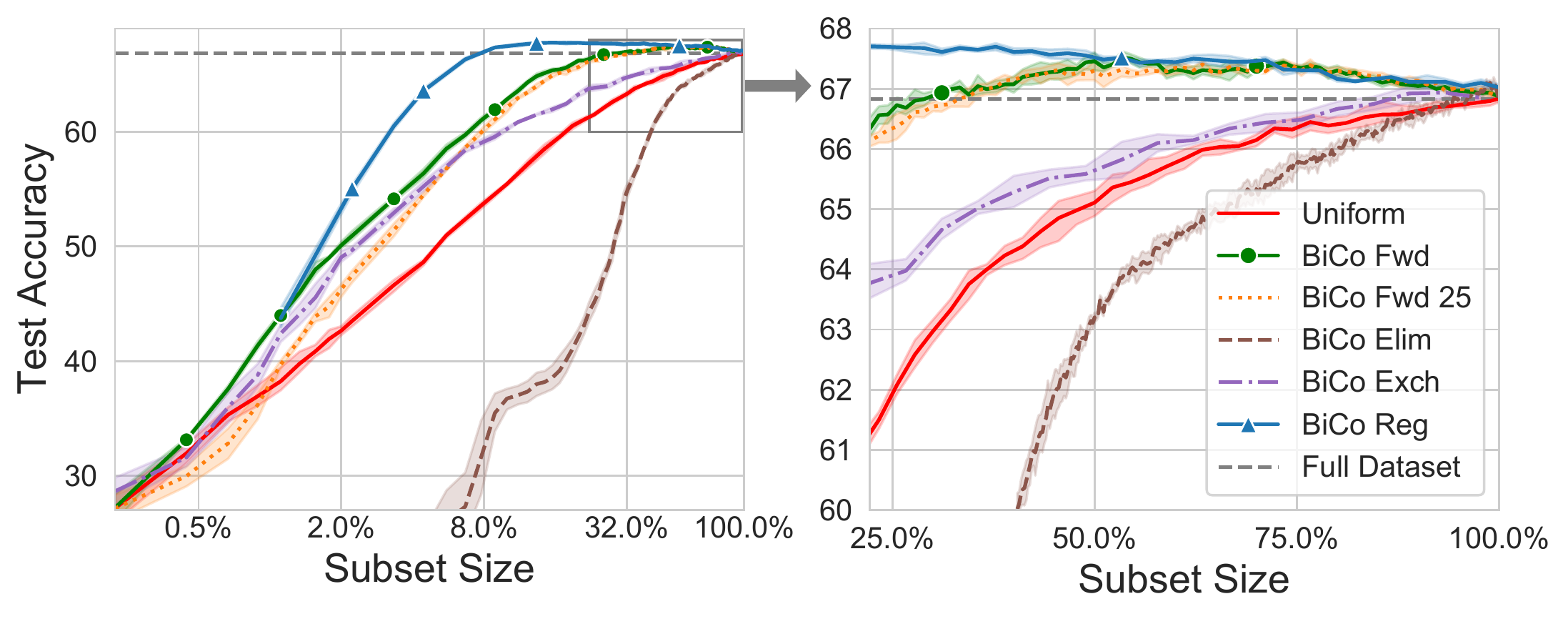}
  \caption{Bilevel coresets for logistic regression on the Nyström feature space of CIFAR-10 CNTK. Building unweighted coresets by forward selection in batches achieves the same performance as one-by-one forward selection after $25\%$  of the points have been selected.  Training on weighted coreset of size $8\%$ of the full data set produced by our method achieves the same performance as training on the full data set. \label{fig:variants}}
  \vspace{-3mm}
\end{figure*}

We study \emph{unweighted} coresets built using one-by-one forward selection, forward selection in batches of 25, exchange, elimination in batches of 200, and exchange with $200$ steps (each step exchanges $1\%$ of the selected points; we found that more steps did not increase the performance). For constructing \emph{weighted} coresets, we solve the regularized version of the bilevel optimization proposed in Section \ref{subsec:variants}. The results are shown in Figure \ref{fig:variants}. We can observe that forward selection in batches  incurs initially a performance penalty  but performs similarly to one-by-one forward selection after $25\%$ of the points have been selected. Both forward selection methods produce coresets of sizes between $33\%$ and $90\%$ on which logistic regression achieves lower test error compared to when trained on the full data set. Similarly to \citet{wang2018data}, we observe that elimination increases the test accuracy in initial iterations; however, it significantly underperforms compared to uniform sampling for generating coresets smaller than $90\%$. Bilevel coresets via regularization (weighted) of size $8\%$ achieve the same performance as training on the full data set. We note that the higher test performance for the weighted coreset with size $20\%$ compared to $90\%$ is due to the higher number of total outer gradient steps performed.

\begin{figure*}[t!]
\centering
  \includegraphics[width=\linewidth]{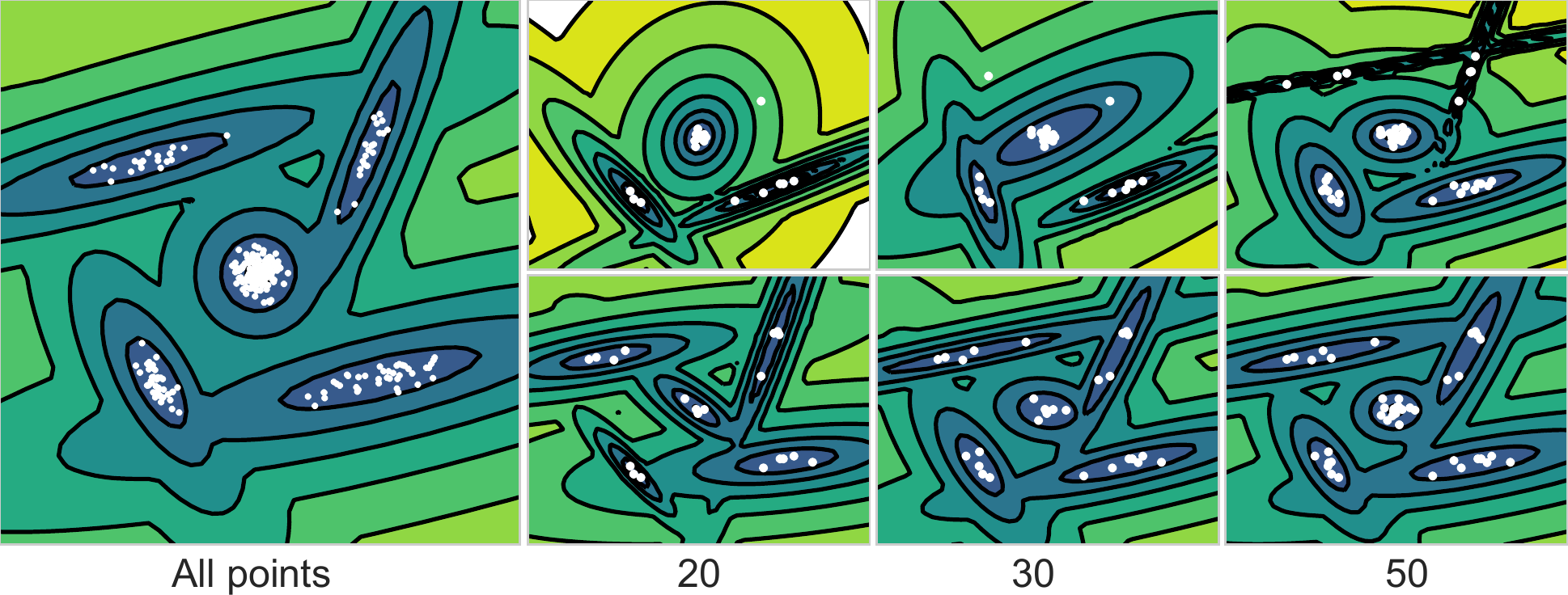}
  \caption{Contours of the log-marginal probability $\log p(x)$ of the Gaussian mixture models (GMM) with $k=5$ components fitted to different subsets of the data. Left: GMM fitted to the full data set; right: GMM fitted to uniform sample (upper row) and to the bilevel coreset (lower row) of sizes indicated by the subscripts. The bilevel coreset provides good approximate density already with $30$ samples.  \label{fig:gmm-contour}}
  \vspace{-3mm}
\end{figure*}

\subsection{Comparison to other Summaries}\label{subsec:comparison}
We compare bilevel coresets to coresets designed for specific models, as well as to other data summarization methods. In all experiments, we observe that other methods do not consistently outperform uniform sampling over all subset sizes in contrast to our method.

\subsubsection{Gaussian Mixture Models}
The first experiment serves as a toy example and proves the versatility of our approach in its broad applicability. In this experiment, we illustrate coreset construction in the \emph{unsupervised} setting of mixture models. We build coresets for Gaussian mixture models with the log marginal probability
\begin{equation*}
    \log p(x) = \log \left( \sum_{i=1}^k \pi_i \mathcal{N}(x|\mu_i, \Sigma_i)\right),
\end{equation*}
where $\{\pi_i\}_{i=1}^k$, $\sum_{i=1}^k \pi_i=1$ are the mixture weights and $\{\mu_i\}_{i=1}^k$ and $\{\Sigma_i\}_{i=1}^k$ are the component means and covariances. The loss function is thus the negative marginal log-likelihood (NLL) $ -\sumin w_i \log p(x_i)$ minimized over the model parameters $\theta:=\{\pi_i,\mu_i, \Sigma_i\}_{i=1}^k$ for the data set $\mathcal{D}=\{x_i\}_{i=1}^n$,
where $w\in \reals_{+}^n$ are the data weights. We generate a synthetic two-dimensional data set so that we can visualize and inspect the choices of the coreset selection. We fit a $k=5$-component Gaussian mixture model to the data by minimizing the NLL using the EM algorithm. To generate the coreset, we use the one-by-one forward selection variant of our algorithm without weight optimization, starting from a random sample of size $10$ and approximate the inverse Hessian-vector product via conjugate gradients.

In Figure \ref{fig:gmm-contour}, we plot the contours of the log-marginal probabilities of the mixtures obtained from fitting the GMM to uniform subsamples and coreset summaries. A coreset of size $30$ already provides accurate mean and covariance estimates, with density contours closely resembling the contours of the model fitted to the full data set. We can observe the following progression of the coreset selection: first, points are picked to represent the modes, after which the component covariance and weight estimates are improved.

\looseness -1 To quantify the improvement obtained by coresets, we measure the relative errors of the negative log-likelihood (NLL) obtained for subsets of different sizes compared to the negative log-likelihood obtained by fitting on the full data set. Furthermore, we also compare to coresets for GMM generated via the sensitivity framework \citep{lucic2017training}. The results in Figure \ref{fig:gmm-re} show an improvement of an order of magnitude by our method.

\begin{figure*}[t!]
\centering
  \includegraphics[width=0.45\linewidth]{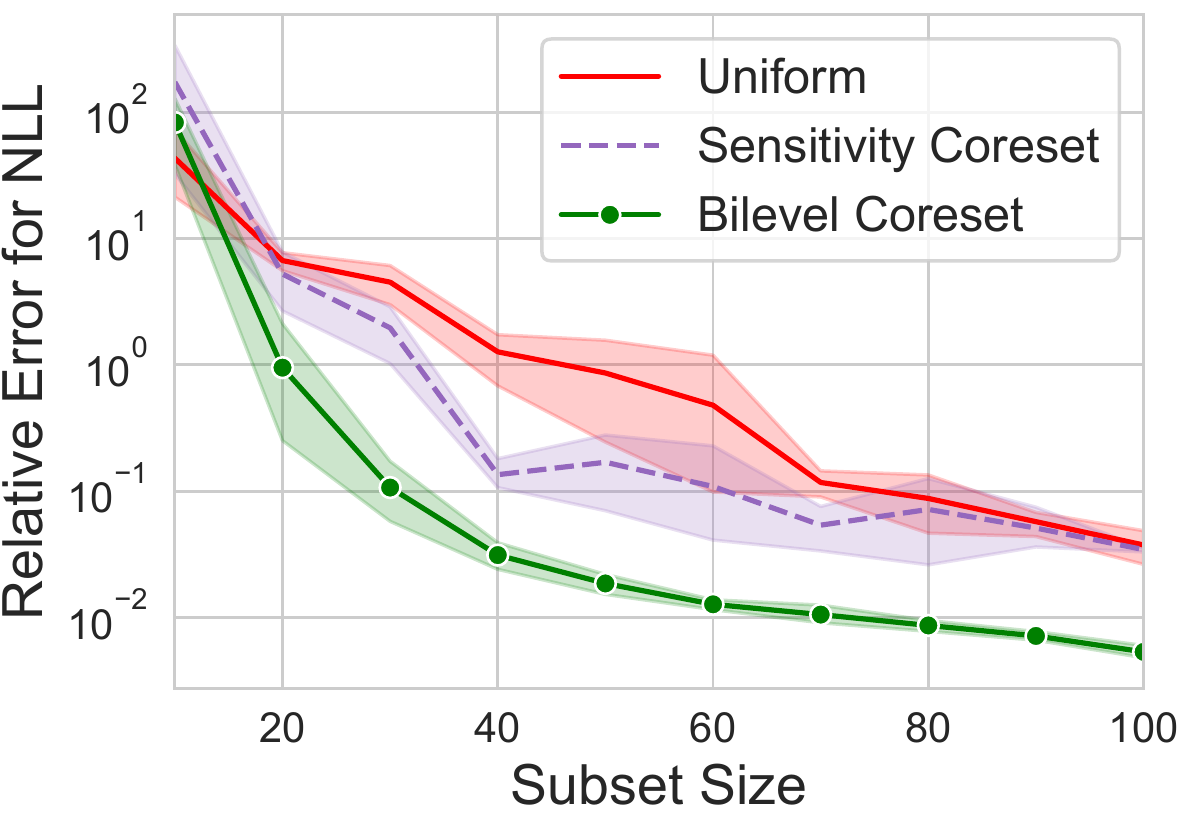}
  \caption{Relative error for the negative log-likelihood of a $k=5$ component GMM obtained by different methods. Our coreset construction outperforms other methods by an order of a magnitude, even those designed specifically for GMMs. \label{fig:gmm-re}}
  \vspace{-3mm}
\end{figure*}

\subsubsection{Logistic Regression} 
\looseness -1 The target model in this  experiment is logistic regression. For binary classification with logistic regression, several coreset constructions have been proposed that serve as our baselines. We choose four standard binary classification data sets \citep{Dua:2019, uzilov2006detection} from the LIBSVM  library\footnote{\href{https://www.csie.ntu.edu.tw/~cjlin/libsvmtools/datasets/binary.html}{https://www.csie.ntu.edu.tw/\texttildelow cjlin/libsvmtools/datasets/binary.html}} for this experiment, of size between $9000$ and $600000$ samples and feature dimensions between $8$ and $123$. We standardize the features and solve the logistic regression on the subsets selected by different methods to compare their test performance with the one achieved by training on the full data set.

\looseness -1 Since the model has low capacity, our framework needs only a small coreset for perfect approximation. Hence, we evaluate the one-by-one forward selection version of our algorithm with weights (Algorithm \ref{alg:bilevel_coreset}, with $150$ outer iterations) and its unweighted version (``BiCo'' and ``BiCo w/ Weights'' in the figures). As for the baselines, we compare to $k$-means++ \citep{arthur2007k} and coresets via sensitivity  \citep{huggins2016coresets}. We also experimented with Hilbert coresets \citep{campbell2019automated}, however, we were unable to tune this method to outperform uniform sampling on these data sets, hence we do not show its performance. We provide detailed description of the baselines in Appendix~ \ref{sec:app-exp-setup}.

\begin{figure*}[t!]
\centering
  \includegraphics[width=0.9\linewidth]{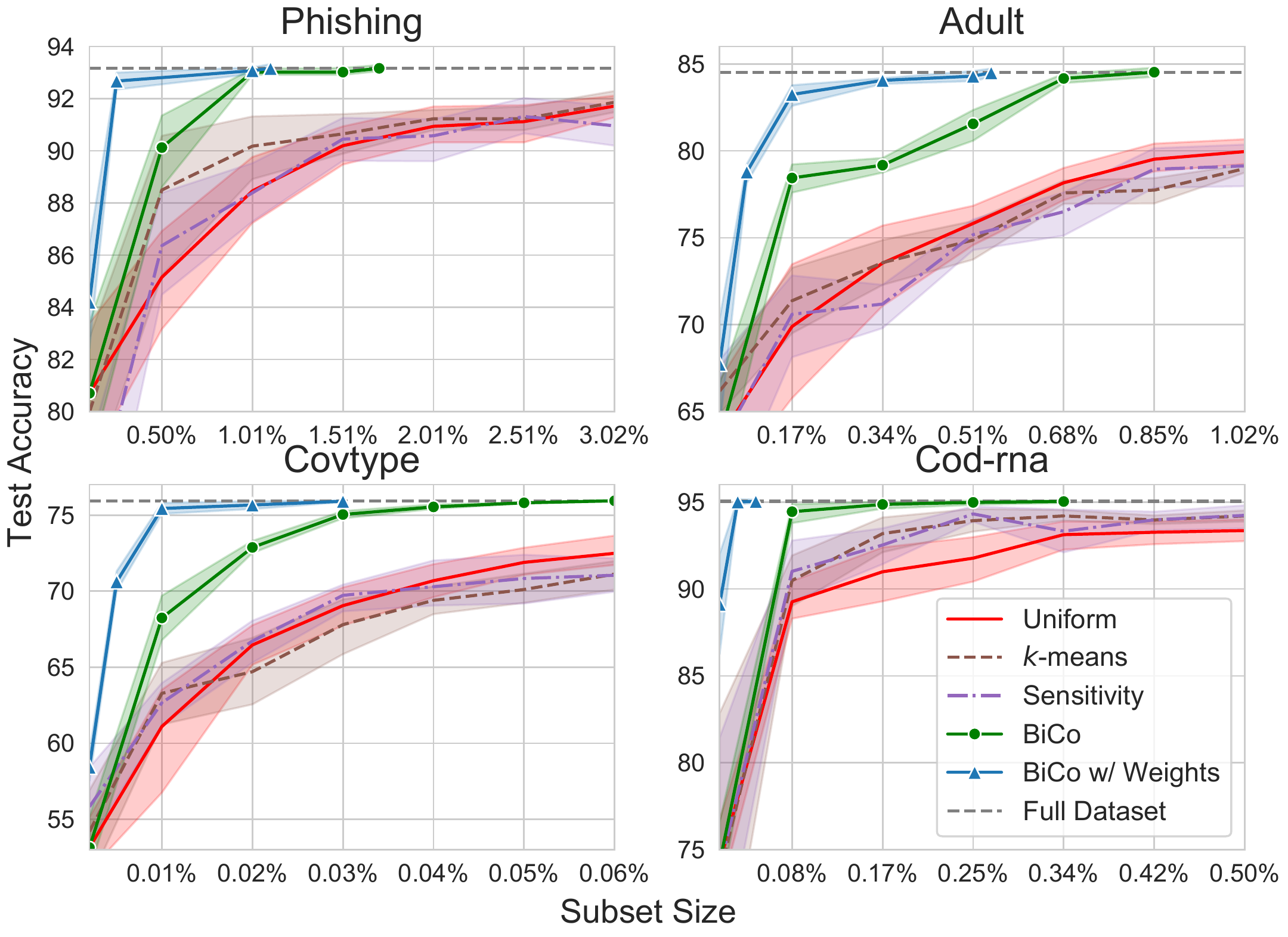}
    \vspace{-2mm}
  \caption{Coresets for binary logistic regression. Our coreset constructions consistently outperform other data summarization approaches, achieving the same performance with $2\%$ of the data as training on the full data set. \label{fig:binary-logistic}}
  \vspace{-3mm}
\end{figure*}

Figure \ref{fig:binary-logistic} shows that our weighted coreset construction needs less than $2\%$ of the data to obtain the same test accuracy as when the model is trained on the full data set. The unweighted variant needs twice as large coreset on average to achieve the same performance, however, it is significantly faster to construct---concretely, it takes $12.2$ seconds on average per data set to construct a coreset of size $100$, which is a factor of $150$ faster than  weighted coreset generation (the speedup factor equals the number of outer iterations). Further details about the experimental setup can be found in Appendix \ref{sec:app-exp-setup}.

\looseness -1 In the following experiment, we investigate coresets for multiclass logistic regression for MNIST and CIFAR-10. For MNIST, we use $500$-dimensional Nyström features to approximate the feature map of the RBF kernel $k(x,y) = \exp(-\gamma \lVert x-y\rVert^2)$ with $\gamma=10^{-3}$. For CIFAR-10, we use the same setup as in Section \ref{subsec:variants}. Figure \ref{fig:multiclass-logistic} shows the comparison of one-by-one unweighted forward selection and weighted coreset generation via regularization (the algorithm is stopped when its test performance drops below the performance of the forward selection variant), to uniform sampling and $k$-means in the feature space (we also compared to $k$-center selection, which underperforms compared to $k$-means), and to the two-stage selection of samples that are most frequently ``forgotten'' during training  (misclassified at some point in training after being classified correctly before; referred to as ``forgetting'' in the figures) \citep{toneva2018an}---we have also experimented with this method in the binary logistic regression experiment, but it significantly underperformed compared to uniform sampling under subset sizes of $300$. Our proposed methods can achieve  compression ratios of over $10$ (data set size divided by smallest data set size required for obtaining the same test accuracy) on both data sets with weighted coresets generated via regularization. 
 
\begin{figure*}[t!]
\centering
\begin{subfigure}[c]{0.49\textwidth}
\centering
  \includegraphics[width=\linewidth]{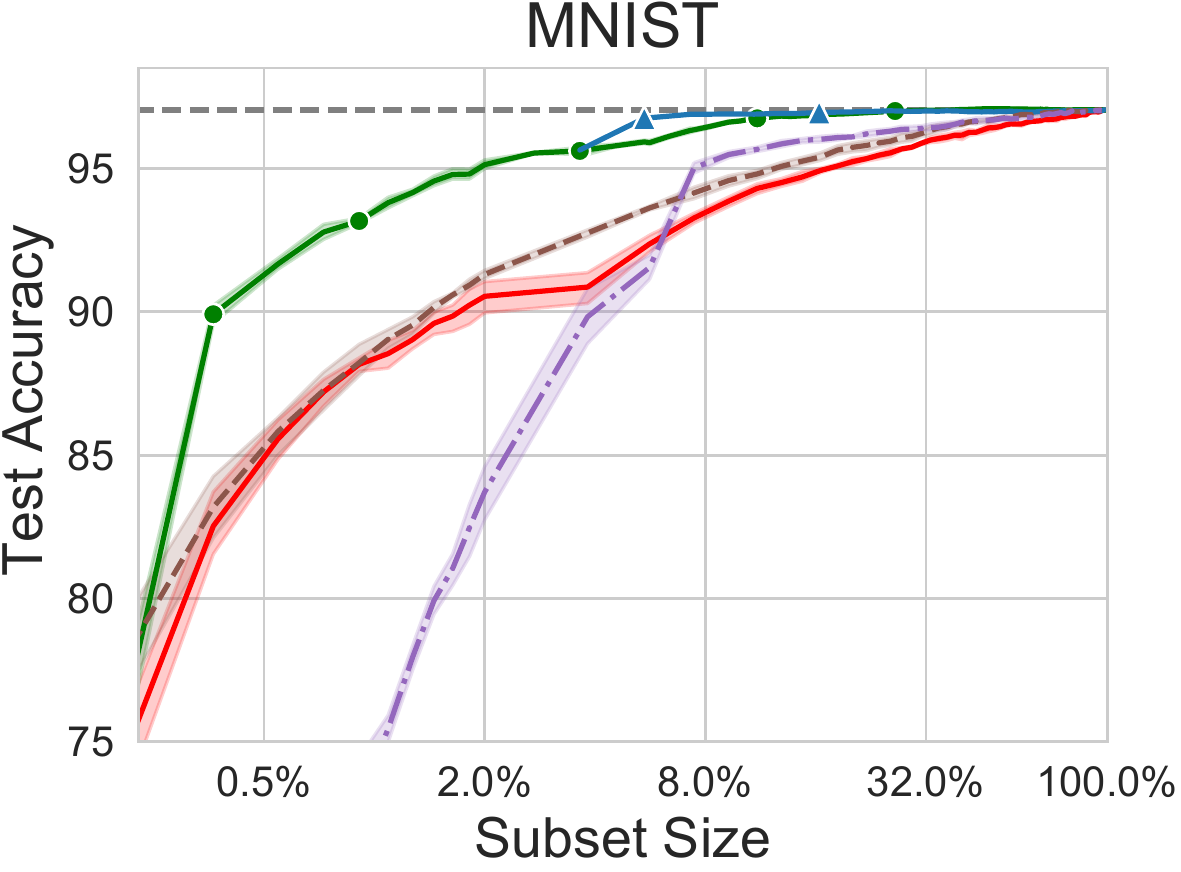}
\end{subfigure}
\begin{subfigure}[c]{0.49\textwidth}
\centering
  \includegraphics[width=\linewidth]{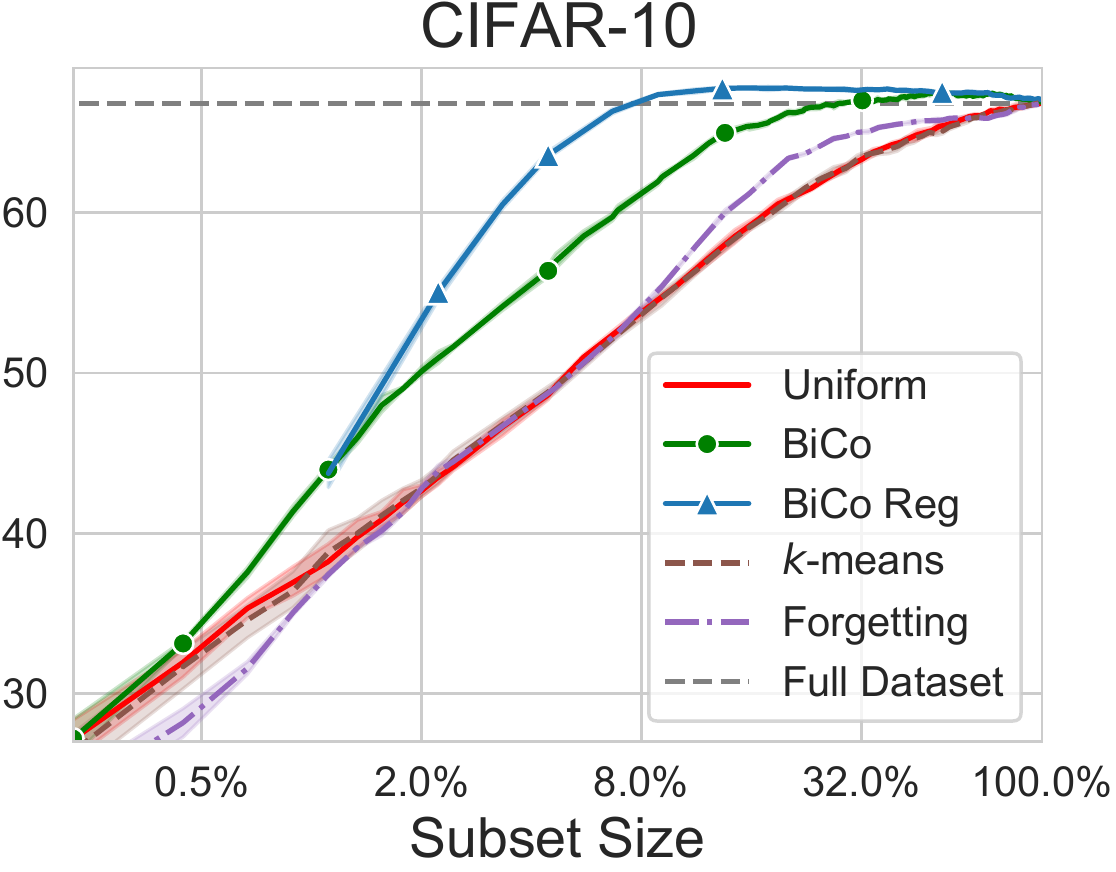}
  \end{subfigure}
\centering
  \caption{Coresets for multiclass logistic regression. Weighted coresets generated via regularization achieve  compression ratios of over $10$ on both data sets. \label{fig:multiclass-logistic}}
\end{figure*}

\subsubsection{Neural Networks} 
For building small coresets ($<1000$) for neural networks, we find that construction via the Neural Tangent Kernel proxy is more effective---the experiments in the following sections concern such settings. For building large coresets, however, working with the original model is more effective. For computationally tractable implicit gradient evaluations for networks with millions of parameters, we evaluate unweighted coreset construction with batch forward selection and inverse Hessian-vector products approximated using the Neumann series 
$\left(\frac{\partial^2 f}{\partial \theta \partial \theta^\top}\right)^{-1} = \lim_{T \to \infty} \sum_{i=0}^T \left(I - \frac{\partial^2 f}{\partial \theta \partial \theta^\top}\right)^i$ \citep{lorraine2019optimizing}.
We truncate the series to $T=100$ terms and we introduce a scaling hyperparameter $\alpha$ for the inner loss $f$, such that the Neumann series approximation is now applied to $\left(\alpha \frac{\partial^2 f}{\partial \theta \partial \theta^\top}\right)^{-1}$. This is to ensure the converge of the Neumann series, for which a necessary and sufficient condition is $\max_j \left| \lambda_j\left(I - \alpha \frac{\partial^2 f}{\partial \theta \partial \theta^\top}\right) \right| < 1$, where $\lambda_j(A)$ denote the $j$-th eigenvalue of $A$ \citep{chen_2005}. In automatic differentiation frameworks, Hessian-vector products can be calculated efficiently without instantiating the Hessian. However, due to memory considerations, we can only afford to evaluate $f$ on a single minibatch of data in the Hessian-vector products, which introduces another layer of approximation through the stochastic Hessian.

\begin{figure*}[t!]
\centering
\begin{subfigure}[c]{0.47\textwidth}
\centering
  \includegraphics[width=\linewidth]{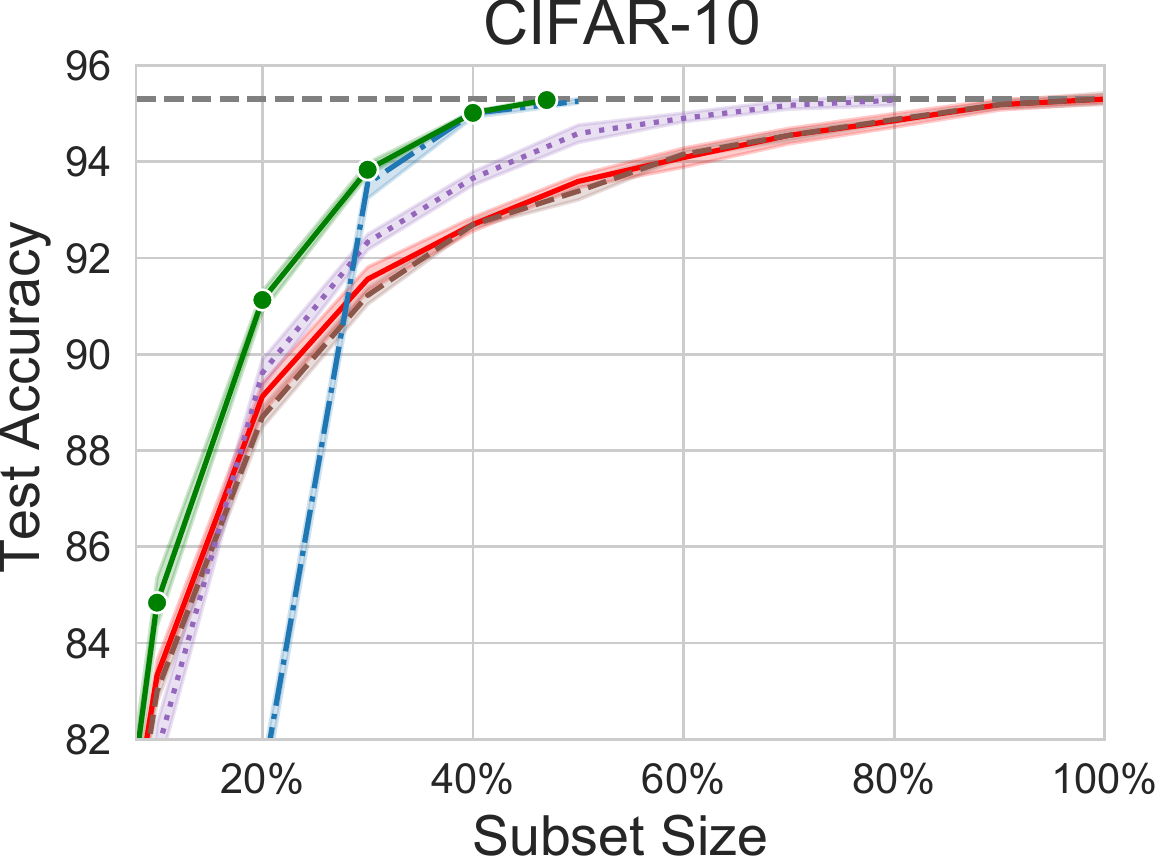}
  \label{fig:no-proxy-cifar}
\end{subfigure}
\begin{subfigure}[c]{0.458\textwidth}
\centering
  \includegraphics[width=\linewidth]{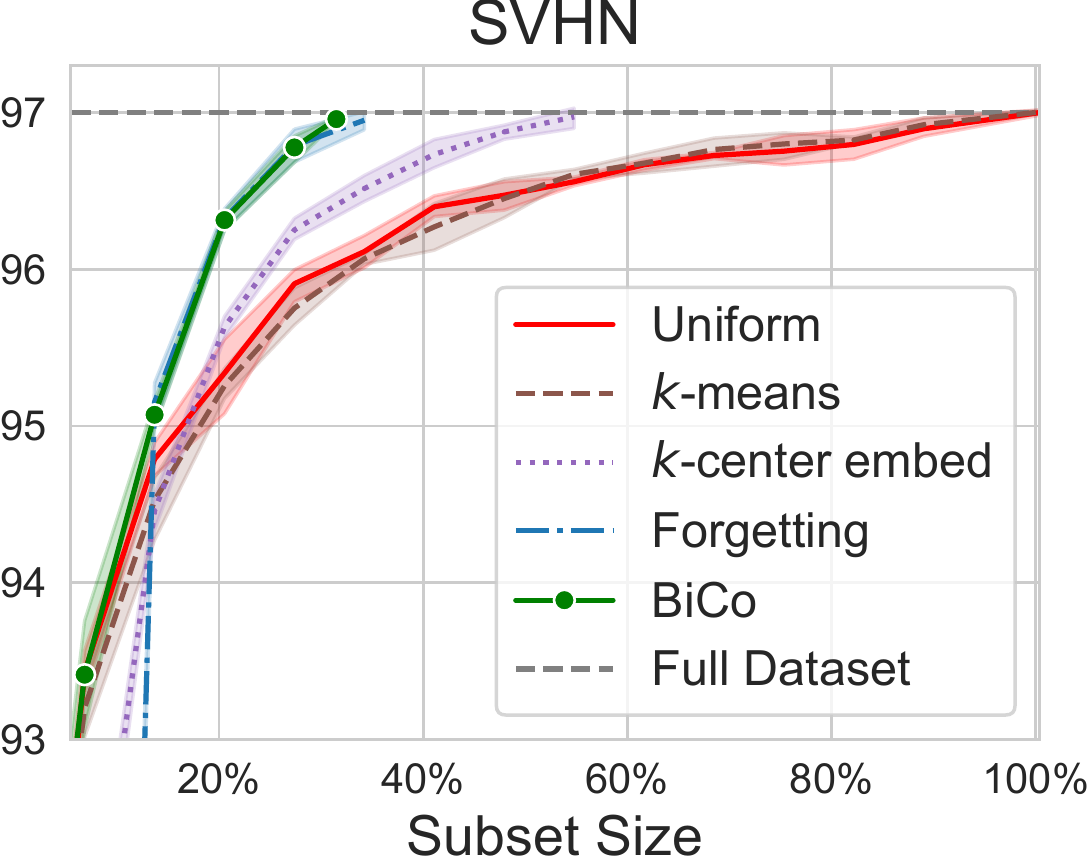}
  \label{fig:no-proxy-svhn}
  \end{subfigure}
\centering
  \caption{Coreset construction with forward batch selection for WideResNet-16-4; bilevel coresets achieve compression ratios of $2$ and $3$.\label{fig:no-proxy}}
\end{figure*}

\looseness -1 We demonstrate the effectiveness of bilevel coresets for  wide residual networks \citep{zagoruyko2016wide} and search for the smallest coreset size such that the test performance matches the test performance of training on the full data set up to a $0.05\%$  tolerance. For computational considerations, we showcase the unweighted coreset construction via forward selection in batches of $250$ points, starting from a random pool of $2500$ points. We evaluate the method by constructing coresets for a WideResNet-16-4 ($2.7$ million parameters) on CIFAR-10 and SVHN \citep{svhn}---for SVHN we only use the train split, containing approximately $73000$ images. We achieved the best results by retraining the network from scratch after every round of selection with SGD with momentum---further details about the training can be found in Appendix \ref{sec:app-exp-setup}.

\looseness -1 We compare in Figure \ref{fig:no-proxy} our unweighted batch forward selection to the following  subset selection methods for neural networks: uniform sampling, $k$-means/$k$-center in the pixel space \citep{v.2018variational}, $k$-means/$k$-center in the last layer embedding of the trained network \citep{sener2018active}, and selecting samples that are most frequently ``forgotten'' during training \citep{toneva2018an}.  We plot each method's test performance until they first reach the test performance of training on the full data set.

We find that, for both CIFAR-10 and SVHN, $k$-means outperforms $k$-center in the feature space, while $k$-center is better for selection based on the last layer embeddings. The performance of $k$-means/$k$-center suggests that simple definitions of redundancy are suboptimal for constructing coresets.  Figure \ref{fig:no-proxy} shows that our method achieves a compression ratio of $2$ on CIFAR-10 and $3$ on SVHN, i.e., it can find a representative subset of the training data of size $23500$ ($47\%$) for CIFAR-10 and $23000$ ($31\%$) for SVHN, such that the WideResNets trained on the chosen subsets achieve the test performance comparable to training on the full data set (within a $0.05$ margin of $95.30$ for CIFAR-10 and $97.01$ for SVHN).
Whereas retaining points that are frequently forgotten \citep{toneva2018an} matches the performance of our method for coreset sizes above $15000$ ($30\%$) for CIFAR-10 and $10000$ ($14\%$) for SVHN, it underperforms uniform sampling in generating small coresets.

 An important application of data summarization is speeding up \emph{hyperparameter tuning}, since the evaluations can be performed on the summary instead of the full data set. In neural architecture search, highly model-specific summaries are undesirable, as they might favor specific architectural choices. To inspect whether the coresets for WideResNet-16-4 are transferable to other architectures, we measure the performance of VGG16 \citep{Simonyan15} and MobileNetV2 \citep{sandler2018mobilenetv2} adapted to CIFAR-10 and SVHN (kernel strides and pooling kernel sizes reduced to accommodate $32\times32$ images) on coresets  of size $23000$; the training procedure is the same as for the WideResNet. Table \ref{table:coreset-transfer} shows that, whereas the transferred coresets do not reach the full data set performance, they perform significantly better than uniform sampling and the transferred $k$-center summary and perform similarly to the ``forgetting'' summary.

\begin{table}[t!]
\centering
\begin{tabular}{@{}c@{\hspace{0.3\tabcolsep}}c@{\hspace{0.3\tabcolsep}}c@{\hspace{0.3\tabcolsep}}c@{\hspace{0.3\tabcolsep}}c@{}}
\toprule
                                                                                   & \multicolumn{2}{c}{\textbf{CIFAR-10}}   & \multicolumn{2}{c}{\textbf{SVHN}}   \\
                                                                                   & \textbf{VGG16}                                  & \textbf{MobileNetV2}                            & \textbf{VGG16}                                & \textbf{MobileNetV2}                          \\ \midrule
\textbf{Uniform}                                                                   & $91.49 \pm 0.16$                                      & $91.71 \pm 0.33$                                    & $94.24 \pm 0.26$                                & $94.22 \pm 0.39$                             \\
\textbf{$k$-center emb.}                                                          & $92.72 \pm 0.20$                                & $92.77 \pm 0.29$                                & $94.59 \pm 0.21$                              &  $94.83 \pm 0.34$                           \\
\textbf{Forgetting} & $\mathbf{93.75 \pm 0.23}$ & $\mathbf{93.80 \pm 0.08}$ & $\mathbf{95.47 \pm 0.17 }$  & $\mathbf{95.30 \pm 0.16}$ \\
\textbf{BiCo} & $\mathbf{93.66 \pm 0.15}$                                   & $\mathbf{93.65 \pm 0.22}$                            & $\mathbf{95.43 \pm 0.15}$                              & $\mathbf{95.53 \pm 0.20}$                              \\ \midrule
\textbf{Full data set}                                                              & \multicolumn{1}{l}{ $94.23 \pm 0.14$}            & \multicolumn{1}{l}{$94.46 \pm 0.15$}            & \multicolumn{1}{l}{$95.93 \pm 0.07$}          & \multicolumn{1}{l}{$96.04 \pm 0.09$}          \\ \bottomrule
\end{tabular}
    \caption {Coresets of size $23000$ for WideResNet-16-4 transferred to VGG16 and MobileNetV2.  Our method provides similar transfer performance to ``forgetting'' \citep{toneva2018an}, while both outperform other methods.\label{table:coreset-transfer}}
\end{table}

We can improve the transferability of the coreset by building joint coresets for mulitple models, as proposed in Section \ref{subsec:joint-coresets}.
In the following experiment, we generate a joint coreset for WideResNet-16-4 and VGG16 and evaluate the resulting coreset for transferability on MobileNetV2. For this, we use a simple heuristic for approximating the solution of Equation \eqref{eq:bilevel-multiple} with $\lambda=1$: similarly to the previous experiment, we generate the coreset by forward greedy selection in batches of $250$ by alternating the model in each step (i.e., we select a new batch of points for the WideResNet, then for VGG16). The results in Table \ref{table:coreset-transfer-v2} show that this simple heuristic improves the effectiveness of the joint coreset on VGG16 and the transferability to MobileNetV2 at the expense of small performance degradation on WideResNet.

\begin{table}[t!]
\centering
\begin{tabular}{cccc}
\toprule
\textbf{data set}                   & \textbf{Architecture}    & \textbf{\begin{tabular}[c]{@{}c@{}}BiCo\\ WRN\end{tabular}} & \textbf{\begin{tabular}[c]{@{}c@{}}BiCo\\ WRN + VGG\end{tabular}} \\ \midrule
\multirow{3}{*}{\textbf{CIFAR-10}} & \textbf{WideResNet-16-4} & $95.24 \pm 0.06$                                               & $95.18 \pm 0.07$                                                     \\
                                   & \textbf{VGG16}           & $93.66 \pm 0.15$                                               & $93.88 \pm 0.16$                                                     \\
                                   & \textbf{MobileNetV2}     & $93.65 \pm 0.22$                                               & $93.79 \pm 0.18$                                                     \\ \midrule
\multirow{3}{*}{\textbf{SVHN}}     & \textbf{WideResNet-16-4} & $96.97 \pm 0.02$                                               & $96.88 \pm 0.09$                                                     \\
                                   & \textbf{VGG16}           & $95.43 \pm 0.15$                                               & $95.75 \pm 0.14$                                                     \\
                                   & \textbf{MobileNetV2}     & $95.53 \pm 0.20$                                               & $95.76 \pm 0.10$                                                     \\ \bottomrule
\end{tabular}
\caption {Coresets of size $23000$ for WideResNet-16-4 transferred to VGG16 and MobileNetV2 (Coreset WRN); coresets constructed jointly for WideResNet-16-4 and VGG16 transferred to MobileNetV2 (Coreset WRN + VGG). \label{table:coreset-transfer-v2}}
\end{table}

\subsection{Continual Learning} \label{subsec:cl-streaming}
\looseness -1 We compare our approach to existing replay memory management strategies by conducting an extensive experimental study. We focus continual learning setting where the learning algorithm is a neural network, and we keep the network structure fixed during learning on different tasks. This is known as the  ``single-head'' setup, which is more challenging  than instantiating new top layers for different tasks (``multi-head'' setup) and does not assume any knowledge of the task descriptor during training and test time. For validating our coreset construction in the continual learning setting, we use the following $10$-class classification data sets:
\begin{itemize}
    \item PMNIST \citep{goodfellow2013empirical}: consist of $10$ tasks, where in each task all images' pixels undergo the same fixed random permutation.
    \item SMNIST \citep{pmlr-v70-zenke17a}: MNIST is split into $5$ tasks, where each task consists of distinguishing between consecutive image classes. 
    \item SCIFAR-10: similar to SMNIST on CIFAR-10.
\end{itemize}
\looseness -1 Following \citet{NIPS2019_9354}, we keep a subsample of $1000$ points  for each task for all data sets while we retain the full test sets.  For PMNIST, we use a fully connected net with two hidden layers with $100$ units, ReLU activations, and dropout with probability $0.2$ on the hidden layers. For SMNIST and SCIFAR-10, we use  a CNN consisting of two blocks of convolution, dropout, max-pooling, and ReLU activation, where the number of filters are $32$ and $64$ and have size $5\times5$, followed by two fully connected layers of size $128$ and $10$ with dropout. The dropout probability is $0.5$. We fix the replay memory size $m=100$ for tasks derived from MNIST. For  SCIFAR-10, we the  set the memory size to  $m=200$. We train our networks for $400$ epochs using Adam with step size $5\cdot10^{-4}$ after each task. 

\looseness -1 We perform an extensive comparison under the protocol described above of our method to other  data selection methods proposed in the continual learning or the coreset literature---the detailed description of the baselines can be found in Appendix \ref{sec:app-cl-streaming}.  For each method, we report the test accuracy averaged over tasks on the best buffer regularization strength $\beta$. For a fair comparison to other methods in terms of summary generation time, we restrict our coreset selection method in all of the continual learning experiments to forward selection of unweighted coresets via the Nyström proxy method with $q=512$ (Section \ref{subsec:variants})---we use the Neural Tangent Kernel \citep{jacot2018neural} corresponding to the chosen architecture, without dropout and max-pooling  obtained with the library of \citet{neuraltangents2020}.

\begin{table*}[t!]
    \centering
\begin{tabular}{@{}cccc@{}}
\toprule
 \textbf{Method}                       & \textbf{PMNIST}                   & \textbf{SMNIST}                  & \textbf{SCIFAR-10} \\ \midrule Training w/o replay                                               & $73.82 \pm 0.49$                     & $19.90 \pm 0.03$                     & $19.95 \pm 0.02$           \\
  Uniform sampling                                    & $78.46 \pm 0.40$  & $92.80 \pm 0.79$  & $\mathbf{43.22 \pm 0.62}$   \\
  $k$-means of features    & $78.34 \pm 0.49$  & $93.40 \pm 0.56$  & $\mathbf{43.96 \pm 0.78}$  \\
  $k$-means of embeddings & $\mathbf{78.84 \pm 0.82}$  & $93.96 \pm 0.48$  & $\mathbf{44.37 \pm 0.76}$     \\
  $k$-means of grads      & $76.71 \pm 0.68$  & $87.26 \pm 4.08$  & $36.99 \pm 1.30$    \\
  $k$-center of features  & $77.32 \pm 0.47$  & $93.16 \pm 0.96$  &  $36.90 \pm 1.09$   \\
  $k$-center of embeddings & $\mathbf{78.57 \pm 0.58}$  & $93.84 \pm 0.78$  & $40.81 \pm 0.53$    \\
  $k$-center of grads      & $77.57 \pm 1.12$  & $88.76 \pm 1.36$  & $35.11 \pm 1.66$  \\
  Gradient matching  & $78.00 \pm 0.57$  & $92.36 \pm 1.17$  & $\mathbf{43.69 \pm 0.73}$   \\
  Max entropy samples       & $77.13 \pm 0.63$  & $91.30 \pm 2.77$  & $35.31 \pm 1.57$   \\
  Hardest samples            & $76.79 \pm 0.55$  & $89.62 \pm 1.23$  & $32.31 \pm 0.88$  \\
  FRCL's selection & $78.01 \pm 0.44$  & $91.96 \pm 1.75$  &  $\mathbf{43.35 \pm 1.15}$   \\
  iCaRL's selection  & $\mathbf{79.68 \pm 0.41}$  & $93.99 \pm 0.39$  & $\mathbf{44.22 \pm 1.31}$   \\
\textbf{BiCo}                                    & $\mathbf{79.33 \pm 0.51}$ & $\mathbf{95.81 \pm 0.28}$  & $\mathbf{44.51 \pm 1.41}$  \\
 \bottomrule
\end{tabular}
    \caption {Continual learning  with replay memory size of $100$ for versions of MNIST and $200$ for CIFAR-10. We report the average test accuracy over the tasks with one standard deviation over $5$ runs with different random seeds. Our coreset construction performs consistently among the best.   \label{table:cl-streaming-res-full}}
    \end{table*}
    
    \looseness -1 We report the results in Table \ref{table:cl-streaming-res-full}. We note that while several methods outperform uniform sampling on some data sets, only our method is consistently outperforming it on all data sets.  For inspecting the gains obtained by our method over uniform sampling, we plot the final per-task test accuracy on SMNIST in Figure \ref{fig:results-per-task-SMNIST}. We notice that our method's advantage does not come from excelling at one task but rather by representing the majority of tasks better than uniform sampling. In Appendix \ref{sec:app-cl-streaming}, we present a study of the effect of the replay memory size.
    
    \looseness -1  Our method can also be combined with different approaches to continual learning, such as  variational continual learning (VCL) \citep{v.2018variational}. Whereas VCL also relies on data summaries, it was proposed with uniform and $k$-center summaries. We replace these with our coreset construction, and, following \citet{v.2018variational}, we conduct an experiment using a single-headed two-layer network with $256$ units per layer and ReLU activation, where the coreset size is set to $20$ points per task. The results in Table \ref{table:vcl-stream-res} corroborate the advantage of our method over simple selection rules and suggest that VCL can benefit from representative coresets.
    
    \begin{table*}[t]
\begin{minipage}[c]{0.47\textwidth}
    \centering
    \begin{tabular}{@{}ccc@{}}
        \toprule
         \textbf{Method}                       & \textbf{PMNIST}                   & \textbf{SMNIST}                  \\ \midrule
         $k$-center & $85.33 \pm 0.67$  & $65.71 \pm 3.17$  \\
           Uniform & $84.96 \pm 0.17$ & $80.06 \pm 2.19$\\
          \textbf{BiCo} & $\mathbf{86.11 \pm 0.25}$ & $\mathbf{84.62 \pm 0.89}$ \\
       \bottomrule
        \end{tabular}
          \caption {VCL with $20$ points/task. VCL can benefit from our coreset construction.  \label{table:vcl-stream-res}}
        \vspace{-3mm}
\end{minipage}
\hfill
 \begin{minipage}[c]{0.5\textwidth}
    \centering
  \includegraphics[width=0.7\linewidth]{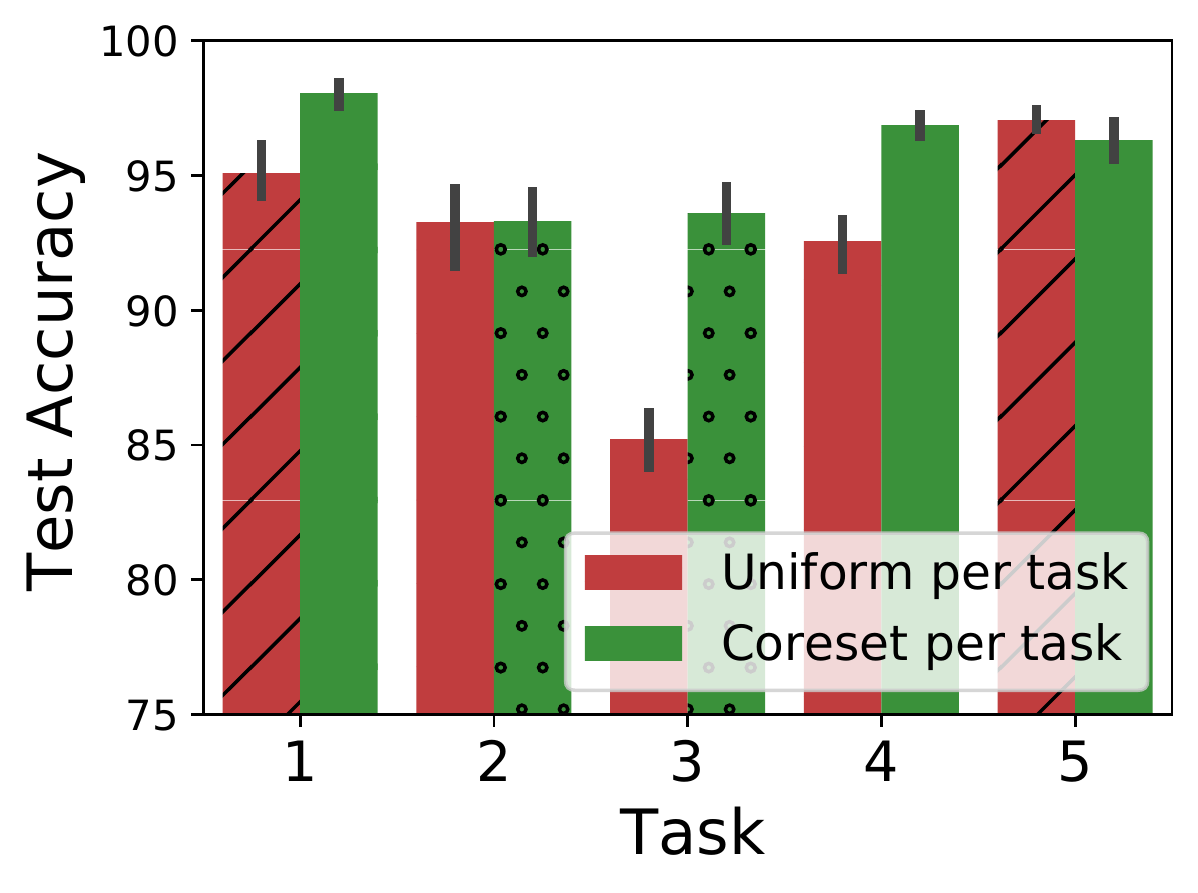}
  \vspace{-3mm}
    \captionof{figure}{Per-task test accuracy on SMNIST.}
    \label{fig:results-per-task-SMNIST}
    \vspace{-3mm}
\end{minipage}

\end{table*}

\subsection{Streaming}
Streaming using neural networks has received little attention. To the best of our knowledge, the replay-based approach to streaming has been tackled by  \citet{NIPS2019_9354}, who propose to select points in the replay memory that maximize the angles between pairs of gradients corresponding to the selected points, \citet{hayes2019memory}, who propose storing cluster centers per class and merging closest clusters for reduction, and \citet{chrysakis2020online}, who propose to class-balance reservoir sampling for imbalanced streams. We compare our method with these methods experimentally.

For evaluating our proposed coreset construction method for training neural networks in the streaming setting, we modify PMNIST and SMNIST by first concatenating all tasks for each data set and then streaming them in batches of size $125$. We fix the replay memory size to $m=100$ and set the number of slots $s=10$---the replay buffer is managed by the merge-reduce framework (Section \ref{subsec:streaming}). We train the models for $40$ gradient descent steps using Adam with step size $5\cdot10^{-4}$ after each batch. We use the same architectures as in the continual learning experiments. 
 
\begin{table*}[t!]
        \begin{minipage}{0.55\textwidth}
            \centering
\begin{tabular}{@{}ccc@{}}
        \toprule
         \textbf{Method}                       & \textbf{PMNIST}                   & \textbf{SMNIST}                  \\ \midrule
         
           Train on coreset  & $45.03 \pm 1.31$ & $89.99 \pm 0.76$\\
             Reservoir                             &  $73.21  \pm 0.59$ & $90.72  \pm 0.97$    \\
             \textbf{BiCo}                                   & $\mathbf{74.49 \pm 0.59}$ & $\mathbf{92.51 \pm 1.30}$  \\
       \bottomrule
        \end{tabular}
            \caption {Streaming with replay memory of size $100$. BiCo uses the merge-reduce buffer. \label{table:stream-res}}
        \end{minipage}
        \begin{minipage}{0.44\textwidth}
    \begin{tabular}{@{}ccc@{}}
        \toprule
          \textbf{Method}                       & \textbf{SMNIST}                   & \textbf{SCIFAR-10}                 \\ \midrule
          Reservoir                                   & $80.60 \pm 4.36$ &  $30.42 \pm 0.93$    \\
        CBRS    & $89.71 \pm 1.31$  & $\mathbf{37.51 \pm 1.15}$   \\
         \textbf{BiCo}  & $\mathbf{92.37 \pm 0.27}$ & $\mathbf{37.09 \pm 0.65}$   \\
       \bottomrule
    \end{tabular}
            \caption {Imbalanced streaming on SMNIST and SCIFAR-10.  \label{table:imbalanced-stream-main}}
        \end{minipage}
    \end{table*}

 We compare our coreset selection to reservoir sampling \citep{vitter1985random} and the sample selection methods of \citet{NIPS2019_9354} and \citet{hayes2019memory}. 
 We were unable to tune the latter two to outperform reservoir sampling, except the gradient-based selection method of \citet{NIPS2019_9354} on PMNIST, achieving a test accuracy of $74.43 \pm 1.02$.
 Table \ref{table:stream-res} shows the dominance of our strategy over reservoir sampling. The table also shows the performance on training only once in the end of the stream on the created coreset, which alone provides strong performance, confirming the merge-reduce framework's validity. We have also experimented with streaming on CIFAR-10 with buffer size $m=200$, where our coreset construction did not outperform reservoir sampling. However, when the task representation in the stream is imbalanced, our method has significant advantages.
 
\looseness -1 The setup of the streaming experiment favors reservoir sampling, as the data in the stream from different classes is balanced. We illustrate the benefit of our method in the more challenging scenario when the class representation is {\em imbalanced}. Similarly to \citet{NIPS2019_9354}, we create imbalanced streams from SMNIST and SCIFAR-10, by retaining $200$ random samples from the first four tasks and $2000$ from the last task. In this setup, reservoir sampling will underrepresent the first tasks. For SMNIST we set the replay buffer size to $m=100$ while for SCIFAR-10 we use $m=200$. We evaluate the test accuracy on the tasks individually, where we do not undersample the test set. We train the same CNN as in the continual learning experiments on the two imbalanced streams and set the number of slots to $s=1$. We compare our method to reservoir sampling and class-balancing reservoir sampling (CBRS) \citep{chrysakis2020online}. The results in Table \ref{table:imbalanced-stream-main} confirm the flexibility of our framework and show that it is competitive with CBRS, which is specifically designed for imbalanced streams.

\subsection{Batch Active Learning} \label{subsec:bal}

We evaluate our proposed method focusing on the audio domain, where semi-supervised batch active learning has not yet been studied to the best of our knowledge. Our first contribution in this section is showing that semi-supervised strategies proposed in the image domain are also highly effective in audio keyword recognition tasks. Then, we show  our batch selection strategy significantly outperforms other acquisition strategies on these tasks. Whereas our strategy is oblivious to the SSL algorithm, we choose MixMatch \citep{berthelot2019mixmatch} as the semi-supervised learning algorithm due to its simplicity, ease of adaptation to the audio domain, and strong empirical performance.

\looseness -1  For demonstrating the effectiveness of SSL and its combination with active learning in the audio domain, we focus on the Spoken Digit data set \citep{spokendigit} ($2700$ utterances, $10$ classes) and Speech Commands V2 \citep{speechcommandsv2} ($85000$ utterances, $35$ classes) data sets, both containing utterances of the length of one second or shorter. With the goal of applying deep neural network architectures from the image domain with minimal adaptations, we map the utterances to $32\times32$ mel spectrograms by first resampling them to $16$kHz and applying the mel feature extraction with of window length of $128$ ms, hop length of $32$ ms and $32$ bins.    

We first investigate the advantages of data augmentation and semi-supervised learning. Our model is a Wide ResNet-28-10 \citep{zagoruyko2016wide} with weight decay of $10^{-4}$ and  without dropout, whereas the SSL algorithm is MixMatch with two augmentations for label guessing and unlabeled cost weight $\lambda_u=10$, with other hyperparameters are set to their defaults \citep{berthelot2019mixmatch}. As for data augmentation, we apply the following transformations in order with $0.5$ probability: i) amplitude change by $a \sim U(0.8, 1.2)$, ii)  audio speed change by  $s \sim U(0.8, 1.2)$, iii) random time shifts by $t$ ms, where $t \sim U(-250, 250)$, iv) mixing in background noise with SNR $r$ dB, where $r \sim U(0, 40)$; we use the noise segments from the Speech Commands data set.  

\begin{table}[t!]
\centering
\begin{tabular}{@{}cccc@{}}
\toprule
            & \multicolumn{3}{c}{\textbf{Spoken Digit Nr. of Labeled Samples}}      \\
\textbf{Method}               & $\mathbf{10}$       & $\mathbf{30}$       & $\mathbf{60}$      \\ \midrule
\textbf{Supervised w/o augm.} & $17.33 \pm 3.03$  & $35.17 \pm 9.74$  & $54.56 \pm 5.67$ \\
\textbf{Supervised w augm.}   & $\mathbf{44.06 \pm 6.98}$  & $\mathbf{63.33 \pm 5.25}$  & $79.17 \pm 3.17$ \\
\textbf{MixMatch}             & $\mathbf{55.78 \pm 11.88}$ & $\mathbf{72.67 \pm 10.46}$ & $\mathbf{87.83 \pm 3.91}$ \\ \bottomrule
         &  \multicolumn{3}{c}{\textbf{Speech Commands Nr. of Labeled Samples}}    \\
               & $\mathbf{50}$      & $\mathbf{100}$     & $\mathbf{200}$     \\ \midrule
\textbf{Supervised w/o augm.} & $6.30 \pm 0.66$  & $12.03 \pm 2.01$ & $34.20 \pm 0.91$ \\
\textbf{Supervised w augm.}   & $23.26 \pm 2.27$ & $36.94 \pm 1.87$ & $54.34 \pm 1.46$ \\
\textbf{MixMatch}             & $\mathbf{56.19 \pm 3.02}$ & $\mathbf{74.52 \pm 5.24}$ & $\mathbf{87.94 \pm 2.70}$ \\
\bottomrule
\end{tabular}
\caption {Supervised and semi-supervised learning with uniformly chosen labeled subsets of Spoken Digit  \citep{spokendigit} and Speech Commands data set \citep{speechcommandsv2}. \label{table:fsd-results}}
\end{table}

\looseness -1 We train the models with Adam with initial learning rate $10^{-3}$ cosine annealed to $0$ over $30$ epochs. The results in Table \ref{table:fsd-results}   demonstrate the superiority of semi-supervised learning via MixMatch on the chosen keyword recognition tasks. These results are also strong indicators for the necessity of evaluating batch active learning in the semi-supervised setting. For this, we compare our proposed method with batch selection strategies compatible with semi-supervised learning. We note that some batch selection strategies not applicable with SSL: prominent examples are Bayesian techniques since the common semi-supervised losses do not have Bayesian interpretations. To this end, we implement uniform subsampling, max-entropy selection (predictions averaged over two augmentations), selection based on the $k$-center algorithm in the last layer of the trained network \citep{sener2018active}, the consistency-based batch selection of \citet{gao2019consistency} (with five augmentations for calculating the variance), and BADGE \citep{Ash2020Deep}, that selects the batch based on the $k$-means centers of the last layer gradient embeddings of the hard pseudo-labeled $\du$. 

For our proposed method, we solve the coreset selection problem in Equation \eqref{eq:active-learning-bilevel} with the CNTK proxy with cross-entropy loss (Section \ref{subsec:variants}) with $2048$-dimensional features and we add $10^{-4}$ $L_2$ penalty to the inner objective, turning it into strongly convex multiclass logistic regression problem. Furthermore, we use data augmentations for the inner problem: for each labeled point, we presample $100$ augmentations and concatenate them for batch gradient descent. We perform one-by-one forward selection, with approximate implicit gradients obtained using $100$ steps of conjugate gradients to generate the unweighted coreset. 

\begin{figure*}[t!]
\centering
 \hspace{5mm} 
\begin{subfigure}[c]{\textwidth}
\begin{subfigure}[c]{0.49\textwidth}
\centering
  \includegraphics[width=\linewidth]{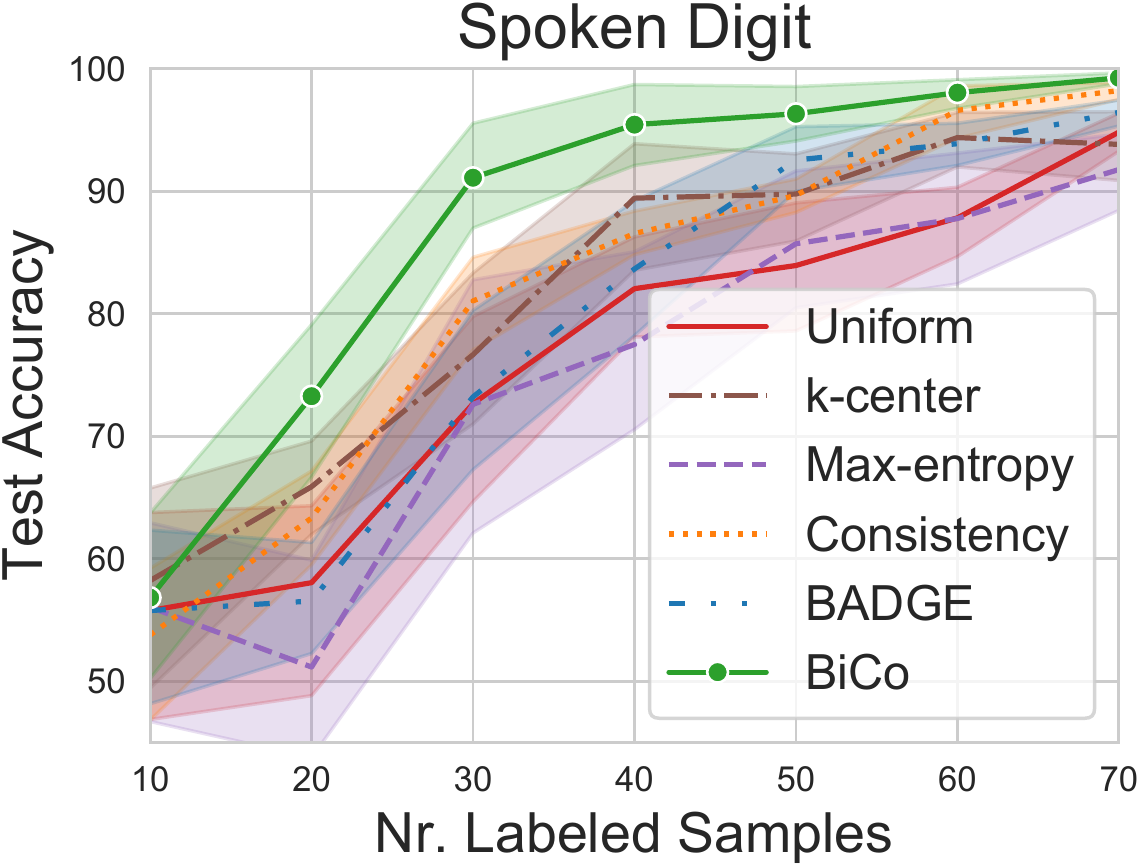}
   \hspace{5mm} 
\end{subfigure}
\begin{subfigure}[c]{0.49\textwidth}
\centering
  \includegraphics[width=\linewidth]{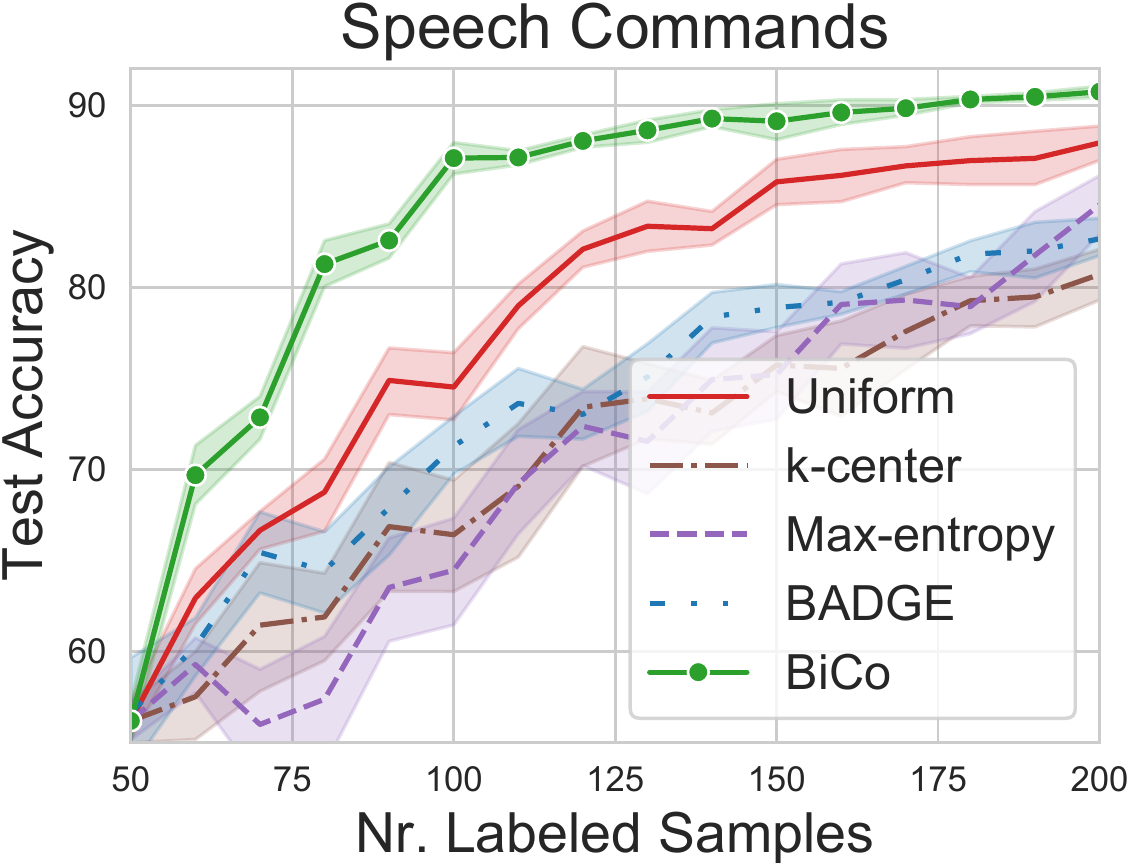}
   \hspace{5mm} 
  \end{subfigure}
\end{subfigure}
\begin{subfigure}[c]{\textwidth}
\centering
  \begin{tabular}{ccccc}
\hline
                  & \multicolumn{4}{c}{\textbf{Nr. of Labeled Samples}}                                  \\
                  & \multicolumn{2}{c}{\textbf{Spoken Digit}} & \multicolumn{2}{c}{\textbf{Speech Commands}} \\
\textbf{Method}   & $\mathbf{40}$         & $\mathbf{70}$        & $\mathbf{100}$        & $\mathbf{200}$        \\ \hline
\textbf{Uniform}  & $82.06 \pm 5.31$    & $94.83 \pm 2.09$   & $74.52 \pm 5.24$    & $87.94 \pm 2.70$    \\
\textbf{Max-ent.} & $77.50 \pm 10.23$   & $91.78 \pm 4.40$   & $64.46 \pm 8.44$    & $84.53 \pm 4.34$    \\
\textbf{k-center} & $\mathbf{89.44 \pm 7.50}$    & $93.83 \pm 3.92$   & $66.41 \pm 8.76$    & $80.71 \pm 3.86$    \\
\textbf{Consist.} & $86.56 \pm 2.33$    & $\mathbf{98.22 \pm 0.96}$   & -                   & -                   \\
\textbf{BADGE} & $83.67 \pm 7.72$    & $96.44 \pm 1.46$   & $71.28 \pm 4.66$    & $82.66 \pm 2.90$    \\
\textbf{BiCo}  & $\mathbf{95.44 \pm 4.67}$    & $\mathbf{99.27 \pm 0.60}$   & $\mathbf{87.10 \pm 2.39}$    & $\mathbf{90.74 \pm 0.85}$    \\ \hline
\end{tabular}

\end{subfigure}
\centering
  \caption{Batch active learning with batch size $m=10$ under semi-supervised training with MixMatch \citep{berthelot2019mixmatch}. Results averaged over six random seeds, shaded areas represent one standard deviation. Our method provides a significant advantage with a small labeled pool.  \label{fig:batch-al}}
\end{figure*}

\begin{figure*}[t!]
\centering
  \includegraphics[width=0.6\linewidth]{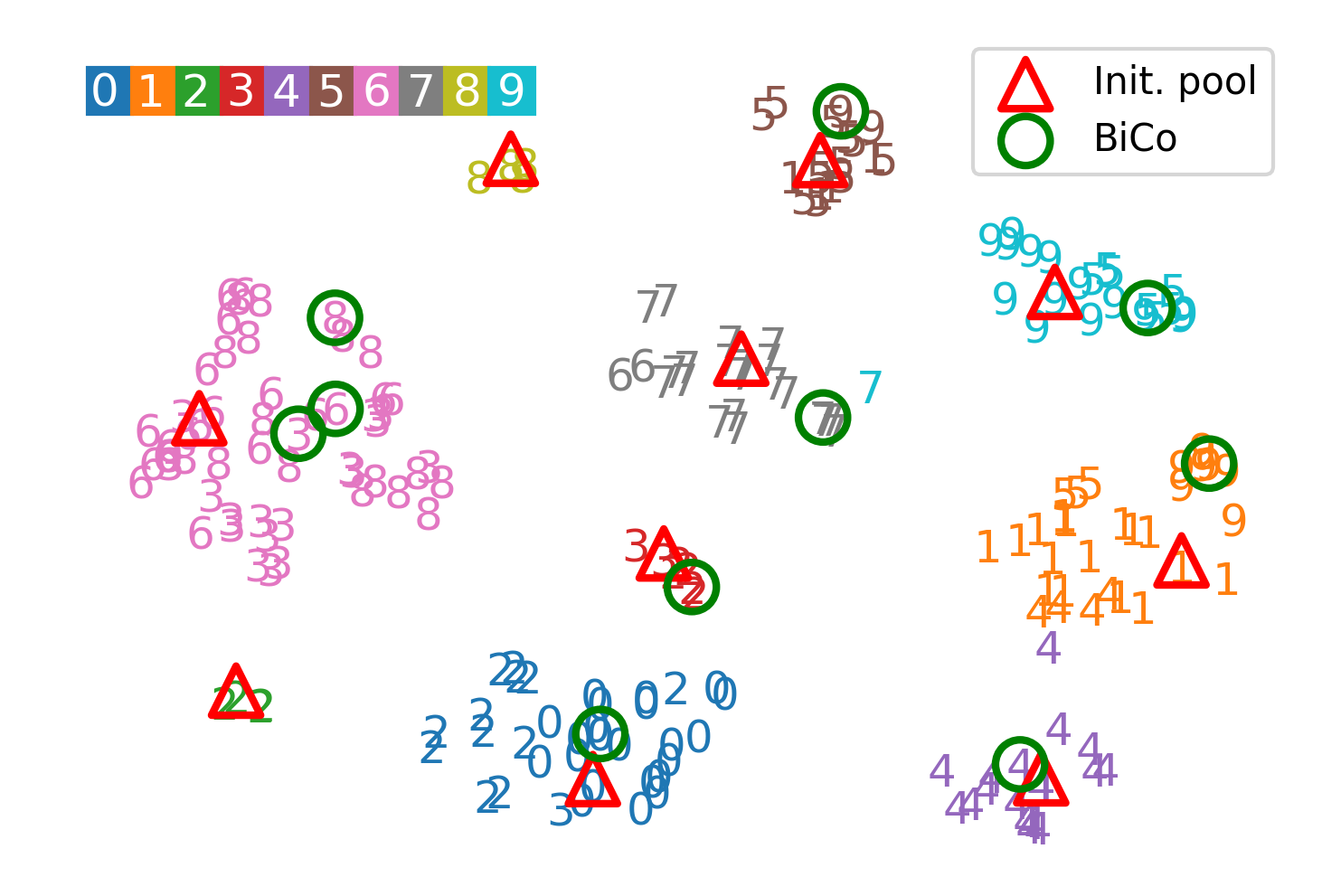}
  \vspace{-2mm}
  \caption{The selected batch of samples for label query by our method (green circles) in the first round of active learning on Spoken Digit data set, when the model is trained on the initial pool (red triangles). The digit values denote the true classes, whereas colors denote predicted classes. The points chosen by our method represent a diverse batch where $6$ out of $10$ points are misclassified. \label{fig:tsne}} 
  \vspace{-3mm}
\end{figure*}

We compare the methods in the challenging setting of small labeled pools ($n_\textup{labeled} \leq 200 $) and perform the acquisition in batches of size $m=10$ starting with $10$ and $50$ labeled samples for  Spoken Digit and Speech Commands, respectively. The starting labeled pools are guaranteed to contain at least one sample from each class. We retrain the models from scratch in every active learning round.

The results in Figure \ref{fig:batch-al} show a significant advantage of our method over other acquisition strategies with only a small number of labeled samples. Especially for Speech Commands, some acquisition strategies suffer from redundancy in the selected batch and, consequently, underperform compared to uniform sampling. We were unable to achieve good performance with the consistency-based acquisition \citep{gao2019consistency} on Speech Commands---this phenomenon was also observed by the authors when starting the method with only a few labeled samples, who refer to it as ``cold start failure''. We also evaluated starting the consistency-based acquisition after a larger number of labeled samples have been acquired by uniform sampling, but the method did not outperform uniform sampling.

\looseness -1 We can gain insight into our proposed method by inspecting the chosen batches of points in the active learning round. For this, we plot the acquisitions in the first round of active learning on Spoken Digit data set in Figure \ref{fig:tsne}, which represents points by their last layer embeddings (after training the network on the initial pool with $60.33\%$ test accuracy) mapped to two dimensions by t-SNE \citep{maaten2008visualizing}. The points chosen by our method represent a diverse batch where $6$ out of $10$ points are misclassified.

\subsection{Dictionary Selection for Compressed Sensing} \label{subsec:dict-sel-exp} In this section, we showcase our framework for selecting dictionary measurement adaptively and incrementally on two examples in Figure \ref{fig:compressed_sensing}: a \emph{synthetic} data set containing a set of random sparse vectors, and the recovery of MNIST digits using a variational autoencoder as the generative model for the images. The baseline algorithms are randomly sampled measurements with normally distributed entries, which satisfy the RIP property with high probability, and approximate-greedy, which is inspired by the heuristics of \citet{icml2010_073} to speed up the greedy algorithm by picking the measurements with the largest average inner product between the signal and the measurements. The classical greedy algorithm is too expensive given the dictionary sizes used here.  The dictionary of linear measurements is chosen as a set of random matrices with entries distributed according to the unit normal distribution, or a wavelets basis for MNIST as done by  \citet{bora2017compressed}, which is a more challenging baseline since not necessarily all elements are equally sparse. In most cases, the compression ratio significantly improves when using our bilevel method.

\subsection{Computational Cost} \label{subsec:comp-cost}

The time complexity of our algorithm depends on the variant and the  model. We now discuss the case of batch forward selection for neural networks. The time complexity depends linearly i) on the number of outer iterations, which equals the coreset size divided by the forward step batch size ($mb^{-1}$)---additionally,  times the number of weight optimization rounds per step if weighting is used ($t_w$);  ii) the number of iterations for solving the inner problem in each step ($t_\textup{inner}$); iii) the complexity of gradient calculations ($g$); iv) the truncation in the Neumann series/the number of conjugate gradient iterations ($t_H$); resulting in the total time complexity $\mathcal{O}(mb^{-1}g(t_\textup{inner} + t_H))$ for unweighted bilevel coresets.

\looseness -1 The proxy reformulation reduces the number of parameters to the order of the dimension of the Nyström features $\mathcal{O}(q)$.  On the other hand, the proxy reformulation introduces the overhead of calculating the proxy kernel, which might be a significant overhead for deep neural networks. We measure the time for generating coresets with the CNTK Nyström proxy with $q=512$ from a data set of $1000$ points for the small CNN (SCNN) described in Section \ref{subsec:cl-streaming} and for the WideResNet-16-4 from Section \ref{subsec:comparison}. We calculate the corresponding NTKs without batch normalization and pooling with the library of \citet{neuraltangents2020} on a single GeForce GTX $1080$ Ti GPU, whereas the coreset selection is performed on a single CPU. The results are shown in Table \ref{table:runtimes}.

\looseness -1  Without the proxy reformulation, a single implicit gradient calculation (Equation \ref{eq:implicit-gradient}) incurs the cost of calculating  $\partial g / \partial \theta$  and $\partial^2 f /\partial \theta \partial w^\top$ and the cost of the Neumann series approximation. Each of these operations has to be performed in minibatches, requiring multiple backpropagation steps. We measure the cost of these operations for WideResNet-16-4 in Table \ref{table:runtimes}, totaling to two minutes per implicit gradient calculation---for reference, we need $84$ implicit gradient calculations for generating the coreset of size $23500$ for CIFAR-10 in Figure \ref{fig:no-proxy}.

Another important consideration both in the proxy and the standard form is how to solve the inner optimization problem after an implicit gradient step. In all our applications except for deep neural networks Section \ref{subsec:comparison}, we resume the inner optimization after the gradient update with the optimal parameters found before the update and perform a small number of inner update steps. For deep neural networks trained with learning rate schedules, we find it beneficial to retrain our models from scratch after each forward batch selection step. A promising future direction is to speed up the selection process without a proxy for neural networks by eliminating the need for retraining from scratch. 

\begin{figure}[t!]
    \centering
    \includegraphics[width=\textwidth]{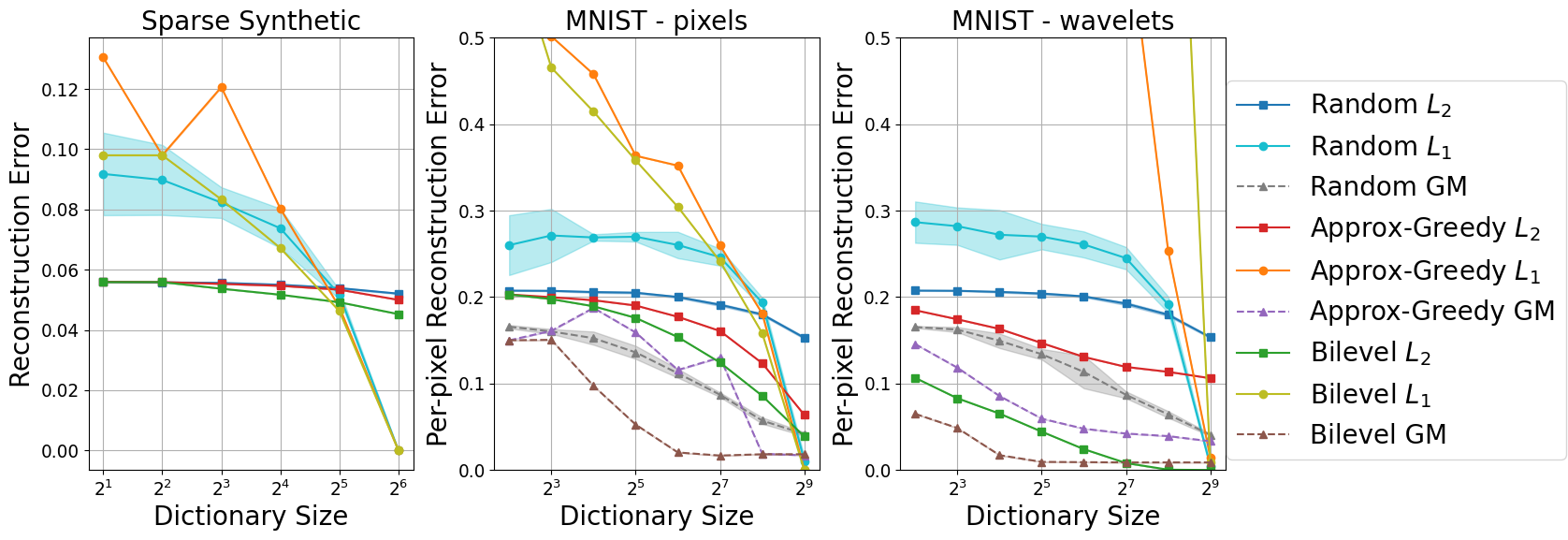}
    \caption{Reconstruction error over the dictionary of signals plotted against the subset size of measurements. The recovery methods $L_2$, $L_1$ and generative model (GM) are compared with different algorithms to generate coreset of measurements for recovery: random, approx-greedy and bilevel. The size the dictionary is 16384, 786 and 786 from left to right. }
    \label{fig:compressed_sensing}
\end{figure}

\begin{table}[t!]
\centering
\begin{minipage}[t]{0.65\textwidth}
\centering
\begin{tabular}{@{}ccc@{}}
\toprule
\textbf{Op}                   & \textbf{SCNN} & \textbf{WideResNet-16-4} \\ \midrule
\textbf{NTK calc.} & $5.1$ s                                                          & $40.2$ s                                                            \\ 
\textbf{Coreset $\mathbf{100}$} & $29.8$ s                                                         & $31.0$ s                                                            \\
\textbf{Coreset $\mathbf{400}$} & $150.6$ s                                                         & $154.4$ s                                                            \\\bottomrule
\end{tabular}
\vspace{-2mm}
\end{minipage}
\hfill
\begin{minipage}[t]{0.33\textwidth}
\centering
\begin{tabular}{@{}cc@{}}
\toprule
\textbf{Op}                   & \textbf{Time} \\ \midrule
$\bm{\partial g / \partial \theta}$ & $29.4$ s                     \\ 
\textbf{Neumann s.} & $18.9$ s  \\
$\bm{\partial^2 f /\partial \theta \partial w^\top}$ & $70.8$ s      \\\bottomrule
\end{tabular}
\end{minipage}
\caption {Left: runtimes for generating coresets out of $1000$ points with CNTKs\label{table:runtimes}. Right: a single implicit gradient calculation step for a WideResNet-16-4 with Neumann series approximation with $100$ terms on CIFAR-10.}
\vspace{-2mm}
\end{table}

\newpage

\section{Discussion and Conclusion}
 We presented a generic coreset construction framework applicable to any twice differentiable model without requiring model-specific adaptations.  We proposed several variants for scaling up the basic algorithm to large coresets and large models. We showed that our method is effective for a wide range of models in various settings, outperforming specialized coreset constructions and other data summarization methods.

\subsection{Limitations} We provide guarantees only for the  case where the overall optimization objective $G(w)$ is convex. Due to the hardness of the cardinality-constrained bilevel optimization problem, our method is only a heuristic for the non-convex settings. 
Except for the binary logistic regression experiments or kernelized linear regression, the cost of our proposed coreset construction is higher than the cost of training the model on the full data set. This contrasts with some of the previous coreset constructions' goal to speed up the training process. Our method is thus suited for settings with memory or human resource constraints, as well as when the summary is reused (e.g., in hyperparameter tuning)---settings for which we demonstrated  the effectiveness of our approach empirically in Section \ref{sec:experiments}.

 
\subsection{Future Work} 

\looseness -1 The flexibility of our framework in accommodating different upper and lower level objectives allows for various extensions and applications. While we discussed some in this work, there are several promising directions, e.g., the framework could be extended to Bayesian inference by using objectives from variational inference. Furthermore, the idea of formulating subset selection as a cardinality-constrained bilevel optimization problem is very general and can be applied to problems besides coreset construction. Some notable examples include basis selection for the Nyström approximation, feature selection and neural network pruning.


\acks{This research was supported by the SNSF grant 407540\_167212 through the NRP 75 Big Data program, by the European Research Council (ERC) under the European Union’s Horizon 2020 research, innovation programme grant agreement No 815943, and by the Swiss National Science Foundation through the NCCR Catalysis.}


\newpage

\appendix
\section{Connection to Experimental Design} \label{sec:app-exp-design}
In this section, the weights are assumed to be binary, i.e., $w \in \{0,1\}^n$. We will use a shorthand $X_S$ for matrix where only rows of X whose indices are in $S \subset [n]$ are selected. This will be equivalent to selection done via the diagonal matrix $D(w)$, where $i \in S$ corresponds to $w_i=1$ and zero otherwise.
Additionally, let $\hat{\theta}$ be a minimizer of the following loss,
\begin{equation}\label{eq:inner}
\hat{\theta} = \argmin_\theta \sumin w_i (x_i^\top \theta -y_i)^2+\lambda\sigma^2||\theta||_2^2
\end{equation}
which has the following closed form,
\begin{equation}\label{eq:closed-form}
    \hat{\theta}_S = (X_S^\top X_S + \lambda \sigma^2 I)^{-1}X_S^\top y_S.
\end{equation}

\paragraph{Frequentist Experimental Design.}
Under the assumption that the data follows the linear model $y=X \theta + \epsilon$, where $\epsilon \sim \mathcal{N}(0,\sigma^2)$, we can show that the bilevel coreset framework instantiated with the inner objective \eqref{eq:inner} and $\lambda = 0$, with various choices of outer objectives is related to frequentist optimal experimental design problems. The following propositions show how different outer objectives give rise to different experimental design objectives.

\begin{proposition}[A-experimental design]\label{prop:A} Under the linear regression assumptions and when $g(\hat{\theta})=\frac{1}{2}\Expect_\epsilon\left[\norm{\theta - \hat{\theta}}^2_2\right]$, with the inner objective is equal to $\eqref{eq:inner}$ with $\lambda = 0$, the objective simplifies, 
\[ G(w) = \frac{\sigma^2}{2} \trace((X^\top D(w)X)^{-1}). \]
\end{proposition}

\begin{proof}
    Using the closed form in \eqref{eq:closed-form}, and model assumptions, we see that $\hat{\theta}_S = \theta + (X_S^\top X_S)^{-1}X_S^\top \epsilon_S$. Plugging this into the  outer objective,
    \begin{eqnarray*}
        g(\hat{\theta}) & = & \frac{1}{2}\Expect_\epsilon\left[\norm{\theta - \hat{\theta}_S}^2_2\right]\\
        & = &\frac{1}{2} \Expect_\epsilon\left[ \norm{(X_S^\top X_S)^{-1}X_S^\top\epsilon_S}_2^2\right] \\
        & = & \frac{1}{2} \Expect_\epsilon\left[ \trace\left( \epsilon_S^\top X_S (X_S^\top X_S)^{-2}X_S^\top\epsilon_S    \right) \right] \\
        & = & \frac{\sigma^2}{2} \trace\left( (X_S^\top X_S)^{-1} \right) \\
        & = & \frac{\sigma^2}{2} \trace\left( (X^\top D(w) X)^{-1} \right)
    \end{eqnarray*}
    where in the third line we used the cyclic property of trace and subsequently the normality of $\epsilon$.
\end{proof}

\begin{proposition}[V-experimental design]\label{prop:V} Under the linear regression assumptions and when $g(\hat{\theta})=\frac{1}{2n}\Expect_\epsilon\left[\norm{X\theta - X\hat{\theta}}^2_2\right]$ and the inner objective is equal to $\eqref{eq:inner}$ with $\lambda = 0$, the objective simplifies, 
\[ G(w) = \frac{\sigma^2}{2n} \trace(X(X^\top D(w)X)^{-1}X^\top). \]
\end{proposition}

\begin{proof}
    Using the closed form in \eqref{eq:closed-form}, and model assumptions, we see that $\hat{\theta}_S = \theta + (X_S^\top X_S)^{-1}X_S^\top\epsilon_S$. Plugging this in to the  outer objective $g(\hat{\theta})$,
    \begin{eqnarray*}
       G(w) & = & \frac{1}{2n}\Expect_\epsilon\left[\norm{X\theta - X\hat{\theta}_S}^2_2\right]\\
        & = &\frac{1}{2n} \Expect_\epsilon\left[ \norm{X(X_S^\top X_S)^{-1}X_S^\top\epsilon_S}_2^2\right] \\
        & = & \frac{1}{2n} \Expect_\epsilon\left[ \trace\left( \epsilon_S^\top X_S (X_S^\top X_S)^{-1}X^\top X(X_S^\top X_S)^{-1}X_S^\top\epsilon_S    \right) \right] \\
        & = & \frac{\sigma^2}{2n} \trace\left( X(X_S^\top X_S)^{-1}X^\top \right) \\
        & = & \frac{\sigma^2}{2n} \trace\left( X(X^\top D(w) X)^{-1}X^\top \right)
    \end{eqnarray*}
    where in the third line we used the cyclic property of trace and subsequently the normality of $\epsilon$.
\end{proof}

\paragraph{Infinite data limit.}
The following proposition links the data summarization objective and V-experimental design in the infinite data limit $n\rightarrow \infty$.

\begin{proposition}[Infinite data limit]\label{prop:cvg} Under the linear regression assumptions $y = X \theta + \epsilon$, $\epsilon \sim \mathcal{N}(0,\sigma^2I)$, let $g_V$ be 
    \[g_V(\hat{\theta}) = \frac{1}{2n} \Expect_\epsilon\left[\norm{X\theta - X\hat{\theta}}^2_2\right] \]
    the V-experimental design outer objective, and let the summarization objective be,
    \[g(\hat{\theta}) = \frac{1}{2n}\Expect_\epsilon\left[\sum_{i=1}^{n} (x_i^\top \hat{\theta} - y_i)^2\right]. \]
For all $\hat{\theta}_S$ in Eq. \eqref{eq:closed-form}, we have   \[ \lim_{n \rightarrow \infty}  g(\hat{\theta}_S)- g_V(\hat{\theta}_S)=\frac{\sigma^2}{2}.\]
\end{proposition}
\begin{proof}
Since $y_i = x_i^\top \theta +\epsilon_i$, we have,
   \begin{eqnarray*}
 g(\hat{\theta}_S) &=&  \frac{1}{2n}\Expect_\epsilon\left[\sum_{i=1}^{n} (x_i^\top \hat{\theta}_S -  x_i^\top \theta -\epsilon_i)^2 \right]\\
 &=& \frac{1}{2n} \Expect_\epsilon\left[\norm{X\theta - X\hat{\theta}_S}^2_2\right]  -  \frac{1}{n} \Expect_\epsilon\left[\epsilon^\top(X\hat{\theta}_S - X\theta )\right]  +  \frac{1}{2n} \Expect_\epsilon\left[\norm{\epsilon}_2^2\right] \\
 &=& g_V(\hat{\theta}_S) -  \frac{1}{n} \Expect_\epsilon\left[\epsilon^\top(X\hat{\theta}_S - X\theta )\right] + \frac{\sigma^2}{2} \\  
  &=& g_V(\hat{\theta}_S) -  \frac{1}{n} \Expect_\epsilon\left[\sum_{i \in S}\epsilon_i(x_i^\top\hat{\theta}_S - x_i^\top\theta )\right] -  \frac{1}{n} \Expect_\epsilon\left[\sum_{i \in [n] \setminus S}\epsilon_i(x_i^\top\hat{\theta}_S - x_i^\top\theta )\right]+ \frac{\sigma^2}{2} \\ 
 \end{eqnarray*}
We have $\lim_{n\rightarrow\infty}\frac{1}{n} \Expect_\epsilon\left[\sum_{i \in S}\epsilon_i(x_i^\top\hat{\theta}_S - x_i^\top\theta )\right] = 0$ since $S$ is a finite set. Since $\hat{\theta}_S$ is independent of $\epsilon_i$, $i\in [n]\setminus S$,
\begin{equation*}
    \Expect_\epsilon\left[\sum_{i \in [n] \setminus S}\epsilon_i(x_i^\top\hat{\theta}_S - x_i^\top\theta )\right] = \sum_{i \in [n] \setminus S}\Expect_\epsilon\left[\epsilon_i\right]\Expect_\epsilon\left[x_i^\top\hat{\theta}_S - x_i^\top\theta \right]=0.
\end{equation*}
    As a consequence, as $\lim_{n \rightarrow \infty}  g(\hat{\theta}_S)-g_V(\hat{\theta}_S) =\frac{\sigma^2}{2}$.

\end{proof}

Note that Proposition \ref{prop:cvg} does not imply that our algorithm performs the same steps when used with $g_V$ instead of $g$. It only means that the optimal solutions of the problems are converging to  selections with the same quality in the infinite data limit. 

\paragraph{Bayesian V-Experimental Design.}
Bayesian experimental design \citep{Chaloner1995} can be  incorporated as well into our framework. In Bayesian modeling, the ``true'' parameter $\theta$ is not a fixed value, but instead a sample from a prior distribution $p(\theta)$ and hence a random variable.
Consequently, upon taking into account the random nature of the coefficient vector, we can find an appropriate inner and outer objectives.

\begin{proposition}\label{prop:Bayes-V} 
Under Bayesian linear regression assumptions and where $\theta \sim \mathcal{N}(0,\lambda^{-1} I)$, let the outer objective  \[g_V(\hat{\theta})=\frac{1}{2n}\Expect_{\epsilon,\theta}\left[\norm{X \theta - X\hat{\theta}}^2_2\right],\] where expectation is over the prior as well. Furthermore, let the inner objective be Eq. $\eqref{eq:inner}$ with the same value of $\lambda$, then the overall objective simplifies to

\begin{equation}\label{eq:bayes-V-design}
G(w) = \frac{1}{2n} \trace\left(  X\left(\frac{1}{\sigma^2}X^\top D(w) X + \lambda I\right)^{-1}X^\top \right).
\end{equation}
\end{proposition}
\begin{proof}
    Using the closed form in \eqref{eq:closed-form}, and model assumptions, we see that  $\hat{\theta}_S = (X_S^\top X_S + \lambda \sigma^2 I)^{-1}X_S^\top(X_S\theta + \epsilon_S)$. Plugging this in to the outer objective $g_V(\hat{\theta})$,
    \allowdisplaybreaks
    \begin{eqnarray*}
     G(w) & = & \frac{1}{2n} \Expect_{\epsilon,\theta}\left[ \norm{ X\theta - X\hat{\theta}_S }_2^2 \right] \\
    & = & \frac{1}{2n} \Expect_{\epsilon,\theta}\left[ \norm{X((X_S^\top X_S + \lambda 
    \sigma^2 I)^{-1}X_S^\top(X_S\theta + \epsilon_S)-\theta)}_2^2\right]\\
    & = & \frac{1}{2n} \Expect_{\epsilon,\theta}\left[\norm{X(X_S^\top X_S + \lambda \sigma^2 I)^{-1}X_S^\top\epsilon_S -\sigma^2\lambda X(X_S^\top X_S + \lambda\sigma^2 I)^{-1}\theta }_2^2\right] \\
    & = & \frac{1}{2n} \Expect_{\theta}\left[\norm{\lambda\sigma^2 X(X_S^\top X_S  + \lambda\sigma^2 I)^{-1}\theta}_2^2\right]\\
    & & + \frac{1}{2n} \Expect_\epsilon\left[
    \norm{X(X_S^\top X_S + \lambda \sigma^2 I)^{-1}X_S^\top\epsilon_S}_2^2\right] \\
    & = & \frac{\sigma^2}{2n}
    \trace\left(\lambda\sigma^2(X_S^\top X_S + \lambda\sigma^2 I )^{-1} X^\top X(X_S^\top X_S + \lambda \sigma^2 I )^{-1} \right)\\ & &  + \frac{\sigma^2}{2n}\trace(X_S (X_S^\top X_S + \lambda \sigma^2 I)^{-1} X^\top X(X_S^\top X_S + \lambda\sigma^2 I)^{-1}X_S^\top )\\
    & = & \frac{\sigma^2}{2n} \trace \left( (X_S^\top X_S + \lambda  \sigma^2I )^{-1}X^\top X(X_S^\top X_S + \lambda\sigma^2 I)^{-1} \left( \lambda \sigma^2 I + X_S^\top X_S  \right)
    \right)\\
    & = &  \frac{\sigma^2}{2n} \trace\left(  X(X_S^\top X_S + \lambda \sigma^2 I)^{-1}X^\top \right) \\
    & = & \frac{\sigma^2}{2n} \trace\left(  X(X^\top D(w) X + \lambda \sigma^2 I)^{-1}X^\top \right)
    \end{eqnarray*}
    
    where we used that $\Expect_\epsilon[\epsilon] = 0$, and cyclic property of the trace, and the final results follows by rearranging.

\end{proof}

Similarly to the case of unregularized frequentist experimental design, in the infinite data limit, even the Bayesian objectives share the same optima. The difference here is that the true parameter is no longer a fixed value and we need to integrate over it using the prior. 

\begin{proposition}[identical to Proposition \ref{prop:cvg-main}] \label{prop:cvg-bayes}
 Under the Bayesian linear regression assumptions $y = X \theta + \epsilon$, $\epsilon \sim \mathcal{N}(0,\sigma^2I)$ and $\theta \sim \mathcal{N}(0,\lambda^{-1})$, let $g_V$ be 
    \[g_V(\hat{\theta}) = \frac{1}{2n}\Expect_{\epsilon,\theta}\left[\norm{X\theta - X\hat{\theta}}^2_2\right] \]
    the Bayesian V-experimental design outer objective, and let the summarization objective be,
    \[g(\hat{\theta}) = \frac{1}{2n}\Expect_{\epsilon,\theta}\left[\sum_{i=1}^{n} (x_i^\top \hat{\theta} - y_i)^2\right]. \]
 For all $\hat{\theta}_S$ in Eq. \eqref{eq:closed-form}, we have  \[ \lim_{n \rightarrow \infty}  g(\hat{\theta}_S)- g_V(\hat{\theta}_S)=\frac{\sigma^2}{2}.\]
\end{proposition}
\begin{proof}
The proof follows similarly as in Proposition \ref{prop:cvg}.
\end{proof}

\begin{lemma}\label{lemma:convex}
Assume $\normp{x_i}{2}<L<\infty$ for all $i\in [n]$ and $w\in \reals^n_+$ s.t. $\normp{w}{2}<\infty$. The function \[G(w) = \frac{1}{2n} \trace\left(  X\left(\frac{1}{\sigma^2}X^\top D(w) X + \lambda I\right)^{-1}X^\top \right) \] is convex and smooth in $w$.
\end{lemma}
\begin{proof}
We will show that the Hessian of $G(w)$ is positive semi-definite (PSD) and that the maximum eigenvalue of the Hessian is bounded, which imply the convexity and smoothness of $G(w)$ .

For brevity, we work with 
$\hat{G}(w) =  \trace\left(  X\left(X^\top D(w) X + \lambda\sigma^2 I\right)^{-1}X^\top \right)$ where $\frac{\sigma^2}{2n}\hat{G}(w)=G(w)$. In addition, denote $F(w) = X^\top D(w) X + \lambda\sigma^2 I$ and $F^+(w) = \left(X^\top D(w) X + \lambda\sigma^2 I \right)^{-1}$ s.t. $F(w) F^+(w) = I$. First, we would like to calculate $\frac{\partial \hat{G}(w)}{\partial w_i}$, for which we will use directional derivatives:
\begin{eqnarray*}
D_v \hat{G}(w) & = &\lim_{h \to 0} \frac{\hat{G}(w+hv) - \hat{G}(w)}{h} \\
            & = & \trace \left( X \left(\lim_{h \to 0} \frac{F^{+}(w+hv) - F^{+}(w)}{h} \right)X^\top \right) \\ 
            & = & \trace \left( X \left(\lim_{h \to 0} F^{+}(w+hv) \cdot \frac{F(w) - F(w+hv)}{h} \cdot  F^{+}(w) \right)X^\top \right) \\ 
            & \overset{\textup{def. of } F}{=} & -\trace \left( X \left(\lim_{h \to 0} F^{+}(w+hv) \cdot \frac{\not{h} X^\top D(v)X}{\not{h}} \cdot  F^{+}(w) \right)X^\top \right) \\ 
            & = & -\trace \left( X  F^{+}(w)   X^\top D(v)X  F^{+}(w) X^\top \right) \\
\end{eqnarray*}
To get $\frac{\partial \hat{G}(w)}{\partial w_i}$, we should choose as direction $v_i := (0,\dots, 0,1,0,\dots,0)^\top$ where $1$ is on the $i$-th position. Since $X^\top D(v_i) X = x_ix_i^\top$, we have that:
\begin{eqnarray*}
\frac{\partial \hat{G}(w)}{\partial w_i} = D_{v_i} \hat{G}(w) & = &  -\trace \left( X  F^{+}(w)   x_ix_i^\top F^{+}(w) X^\top \right) \\
&  \overset{\textup{cyclic prop } \trace}{=} &  -x_i^\top F^{+}(w) X^\top X  F^{+}(w) x_i \\
\end{eqnarray*}
We will proceed similarly to get $\frac{\partial^2 \hat{G}(w)}{\partial w_j \partial w_i}$.
\begin{eqnarray*}
D_v \frac{\partial \hat{G}(w)}{\partial w_i} &=& -x_i^\top  \lim_{h \to 0} \frac{F^+(w+hv) X^\top X  F^+(w+hv) - F^+(w) X^\top X  F^+(w) }{h} x_i \\
             &=& x_i^\top \lim_{h \to 0} F^+(w+hv) \cdot \frac{  F(w+hv) F^+(w) X^\top X  }{h} \cdot  F^+(w)x_i \\
             & & - x_i^\top \lim_{h \to 0} F^+(w+hv) \cdot \frac{X^\top X  F^+(w+hv) F(w)  }{h} \cdot  F^+(w)x_i 
\end{eqnarray*}
Now, since, 
\begin{eqnarray*}
F(w+hv) F^+(w)&=& (F(w) + hX^\top D(v)X)  F^+(w)\\
            &=& I + hX^\top D(v)X F^+(w) \\ 
F^+(w+hv) F(w)&=& F^{+}(w+hv) (F(w+hv) - hX^\top D(v)X) \\ &=& I - hF^+(w+hv)X^\top D(v)X 
\end{eqnarray*}
we have 
\begin{eqnarray*}
D_v \frac{\partial \hat{G}(w)}{\partial w_i} &=& x_i^\top  F^+(w) \left(X^\top D(v)X F^+(w)X^\top X  + X^\top X F^+(w)X^\top D(v)X  \right)  F^+(w) x_i \\
&=& 2  x_i^\top  F^+(w)X^\top D(v)X F^+(w)X^\top X F^+(w) x_i
\end{eqnarray*}
Choosing $v_j$ as our directional derivative, we have:
\begin{eqnarray*}
\frac{\partial^2 \hat{G}(w)}{\partial w_j \partial w_i} = D_{v_j} \frac{\partial \hat{G}(w)}{\partial w_i} & = &  2  \left(x_i^\top  F^+(w)x_j\right)\left(x_j^\top F^+(w)X^\top X F^+(w) x_i\right)\\
 & = &  2  \left(x_j^\top  F^+(w)x_i\right)\left(x_j^\top F^+(w)X^\top X F^+(w) x_i\right)
\end{eqnarray*}
from which we can see that we can write the Hessian of $\hat{G}(w)$ in matrix form as:
\begin{equation*}
    \nabla^2_w \hat{G}(w) = 2  \left(X  F^+(w)X^\top\right) \circ \left(X F^+(w)X^\top X F^+(w) X^\top\right)
\end{equation*}
where $\circ$ denotes the Hadamard product. Since $F^+(w)$ is PSD it immediately follows that $X  F^+(w)X^\top$ and $X F^+(w)X^\top X F^+(w) X^\top$ are PSD. Since the Hadamard product of two PSD matrices is PSD due to the \emph{Schur product theorem}, it follows that the Hessian $\nabla^2_w \hat{G}(w) $ is PSD and thus $G(w)$ is convex.

As for smoothness, we need the largest eigenvalue of the Hessian to be bounded:
\begin{eqnarray*}
    \lambda_{\max} (\nabla^2_w \hat{G}(w)) &\leq& \trace  (\nabla^2_w \hat{G}(w))\\
    &=& 2\sumin \left(X  F^+(w)X^\top\right)_{ii} \left(X F^+(w)X^\top X F^+(w) X^\top\right)_{ii} \\
    &=& 2\sumin \left(x_i^\top  F^+(w)x_i\right) \left(x_i^\top  F^+(w)X^\top X F^+(w) x_i \right) \\
    &=&2\sumin \left(x_i^\top  F^+(w)x_i\right) \normp{X F^+(w) x_i}{2}^2 \\ 
    &\leq & 2\sumin   \lambda_{\max} (F^+(w)) \normp{x_i}{2}^2 \normp{X F^+(w) x_i}{2}^2 \\
    &\leq& 2\sumin   \lambda_{\max} (F^+(w)) \normp{x_i}{2}^2 \normp{X}{2}^2 \normp{F^+(w)}{2}^2 \normp{x_i}{2}^2 \\ 
    &= & 2 \lambda_{\max}^3 (F^+(w)) \normp{X}{2}^2 \sumin   \normp{x_i}{2}^4 \\
    &\leq& \frac{2}{\lambda^3\sigma^6}  \normp{X}{2}^2 \sumin   \normp{x_i}{2}^4  \\ 
    &\leq & \frac{2}{\lambda^3\sigma^6}  \normp{X}{F}^2 nL^4  \\
    &\leq& \frac{2n^2L^6}{\lambda^3\sigma^6}, 
\end{eqnarray*}
where in the fifth line we have used the property of the Rayleigh quotient that for any nonzero vector $x$ and self-adjoint matrix $M$ we have that $x^\top M x\leq \lambda_{\max} (M) \normp{x}{2}^2$.
Thus $G$ is $\frac{nL^6}{\lambda^3\sigma^4}$-smooth.

\end{proof}
\section{Connection to Influence Functions} \label{sec:app-influence-functions}

\paragraph{Proof of Proposition \ref{prop:infl-fns}}
\begin{proofarg}{}
Following \citet{koh2017understanding} and using the result of \citet{cook1982residuals}, under twice differentiability and strict convexity of the inner loss, the empirical influence function at $k$ is
\begin{equation} \label{eq:cook-app}
    \frac{\partial \theta^* }{ \partial \eps^\top} \bigg|_{\eps =0} = -\left(\frac{\partial^2 \sumin w^*_{S,i}  \ell_i(\theta^*)}{\partial \theta \partial \theta^\top} \right)^{-1} \nabla_\theta \ell_k(\theta^*).
\end{equation}
Using the chain rule for $\mathcal{I}(k)$:
\begin{align*}
    \mathcal{I}(k)   &= - \frac {\partial \sumin  \ell_i(\theta^*) }{\partial \eps} \bigg|_{\eps =0} \\
    &=- \left(\nabla_\theta \sumin  \ell_i(\theta^*) \right)^\top \frac{\partial \theta^*}{\partial \eps^\top} \bigg|_{\eps =0} \\
    &\stackrel{\textup{Eq. }~\eqref{eq:cook-app}}{=}  \nabla_\theta \ell_k(\theta^*)^\top \left(\frac{\partial^2 \sumin w^*_{S,i}  \ell_i(\theta^*)}{\partial \theta \partial \theta^\top} \right)^{-1}\nabla_\theta \sumin  \ell_i(\theta^*).
\end{align*}
Hence, $\argmax_k \mathcal{I}(k) $ and the selection rule in Equation \eqref{eq:explicit-selection-rule} are the same.
\end{proofarg}

\section{Detailed Experimental Setup for Sections \ref{subsec:variants} and \ref{subsec:comparison}} \label{sec:app-exp-setup}

\paragraph{Variants} All variants in Section \ref{subsec:variants} use $\lambda=10^{-7}$ regularizer in the inner problem. The inner optimization is performed with Adam using step size of $0.01$ as follows: all variants start with an optimization phase on the initial point set with $5 \cdot 10^4$ iterations; then, after each step, an additional $10^4$ GD iterations are performed. We note that performing $10^4$ GD iterations on the entire data set takes $2.3$ seconds on a single GeForce GTX 1080 Ti. 

\paragraph{Binary Logistic Regression} The features of the data sets are standardized to zero mean and unit variance. The logistic regression is solved using batch Adam with step size $0.01$ and $L_2$-penalty of $0.01$. For the bilevel coresets, the selection process is started from $10$ randomly chosen points and the implicit gradients are calculated through $100$ steps of conjugate gradients. For the unweighted version, $50$ gradient descent steps are performed after each selection. For the weighted version, we use Adam with step size $0.01$ to optimize the weights over $150$ outer iterations in each step.

We consider the following baselines:
\begin{itemize}
    \item $k$-means in the feature space, where the chosen subset is the set of centers selected by $k$-means++ \citep{arthur2007k}; we also evaluated $k$-center, which performed worse than $k$-means on all data sets,
    \item coresets for binary logistic regression via sensitivity \citep{huggins2016coresets}, where, for each data set, we choose the best hyperparameter setting from a grid search over $k\in\{ 5, 10, 25 \}$ and $R\in \{ 0.1, 1, 10, 100\}$ --- we refer to \citet{huggins2016coresets} for the details about the hyperparameters $k$ and $R$.
    \item Hilbert coresets \citep{campbell2019automated} solved via Frank-Wolfe \citep{campbell2019automated} and GIGA \citep{Campbell18_ICML} with the norm chosen as the weighted Fisher information distance and with random features of $500$ dimensions. However, we were unable to tune either of these methods to outperform uniform sampling on any of the data sets, hence we do not show their performance.
\end{itemize}

\paragraph{Neural Networks} For training the networks, we use weight decay of $5\cdot 10^{-4}$ and an initial learning rate of $0.1$ cosine-annealed to $0$ over $300 \cdot n/m$ epochs, where $n$ is the full data set size and $m$ is the subset size. Additionally, we use dropout with a rate of $0.4$ for SVHN. For CIFAR-10, we use the standard data augmentation pipeline of random cropping and horizontal flipping, whereas we do not use data augmentation for SVHN.

\section{Continual Learning and Streaming} \label{sec:app-cl-streaming}

\begin{figure}[t!]
         \centering
      \begin{algorithm}[H]
       \caption{Streaming BiCo with Merge-reduce Buffer}
       \label{alg:streaming-coreset}
    \begin{algorithmic}[1]
       \State {\bfseries Input:} stream $S$, number of slots $s$, $\beta$
       \State
       \Procedure{select\_index}{$[(C_1, \beta_1), \dots , (C_{s+1}, \beta_{s+1})]$}
       \If{$s==1$ or $\beta_{s-1} >\beta_s$}
            \State {\bfseries return $s$}
        \Else
            \State {$k = \arg \min_{i \in [1,\dots,s]}  \left( \beta_{i} == \beta_{i+1}\right)$ }
            \State {\bfseries return $k$}
        \EndIf
       \EndProcedure
       
       \State
        \State $\textup{buffer} = [ \,\, ]$
       \State
        \For{$\mathcal{D}_t$ in stream $S$} 
            \State $\mathcal{C}_t = \textup{construct\_coreset}(\mathcal{D}_t)$
            \State $\textup{buffer.append}((\mathcal{C}_t, \beta))$
            \If {$\textup{buffer.size} > s$}
                \State $k = \textup{select\_index(buffer)}$
                \State $\mathcal{C}'=\textup{construct\_coreset}((\mathcal{C}_k, \beta_k),(\mathcal{C}_{k+1}, \beta_{k+1}))$
                \State $\beta' = \beta_k + \beta_{k+1} $
                 \State \textup{delete buffer}$[k+1]$
                \State \textup{buffer}$[k] = (\mathcal{C}', \beta')$
            \EndIf
        \EndFor
    \end{algorithmic}
    \end{algorithm}
    \vspace{-4mm}
\end{figure}

\begin{figure}[t!]
  \centering
   \scalebox{.85}{\input{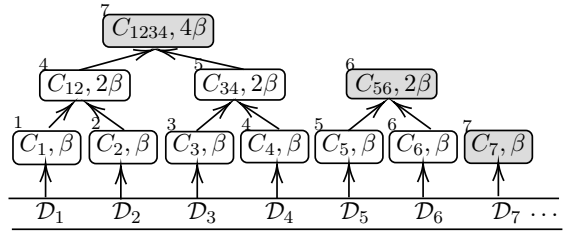}}
    \caption{Merge-reduce on 7 steps with a buffer with 3 slots. The gray nodes are the ones in the buffer after the 7 steps, the numbers in the upper left corners represent the construction time of the corresponding coresets.}
    \label{fig:merge-reduce}
      \vspace{-3mm}
\end{figure}

For the continual learning experiments, we compare the following methods:
\begin{itemize}
    \item Training w/o replay: train after each task without replay memory. Demonstrates how catastrophic forgetting occurs.
    \item Uniform sampling/per task coreset: the network is only trained on the points in the replay memory with different selection methods. 
    \item $k$-means/$k$-center in feature/embedding/gradient space: the per-task selection retains points in the replay memory that are generated by $k$-means++ \citep{arthur2007k}/greedy $k$-center algorithm, where the clustering is done either in the original feature space, in the last layer embedding of the neural network, or in the space of the gradient with respect to the last layer (after training on the respective task). The points that are the cluster centers in different spaces are the ones chosen to be saved in the memory. We note that the $k$-center summarization in the last layer embedding space is the  coreset method proposed for active learning by \citet{sener2018active}.
    \item Hardest/max-entropy samples per task: the saved points have the highest loss after training on each task/have the highest uncertainty (as measured by the entropy of the prediction). Such selection strategies are used, among others, by \citet{coleman2020selection} and \citet{aljundi2019task}.
    \item Training per task with FRCL's/iCaRL's selection: the points per task are selected by FRCL's inducing point selection \citep{titsias2020functional}, where the kernel is chosen as the linear kernel over the last layer embeddings/iCaRL's selection (Algorithm 4 in \citet{rebuffi2017icarl}) performed in the normalized embedding space. 
     \item Gradient matching per task: same as iCaRL's selection, but in the space of gradients with respect to the last layer and jointly over all classes. This is a variant of Hilbert coreset \citep{campbell2019automated} with equal weights, where the Hilbert space norm is chosen to be the squared 2-norm difference of loss gradients with respect to the last layer at the maximum posterior value. 
 \end{itemize}

In the continual learning experiments, we train our networks for $400$ epochs using Adam with step size $5\cdot10^{-4}$ after each task. The loss at each step consists of the loss on a minibatch of size $256$ of the current tasks and loss on the replay memory scaled by $\beta$. For streaming, we train our networks for $40$ gradient descent steps using Adam with step size $5\cdot10^{-4}$ after each batch. For \cite{NIPS2019_9354}, we use a streaming batch size of $10$ for better performance, as indicated in Section 2.4 of the supplementary materials of \cite{NIPS2019_9354}. We tune the replay memory regularization strength $\beta$ separately for each method from $\{0.01, 0.1, 1, 10, 100, 1000\}$ and report the best result on the test set. 

  \begin{table}[h]
    \centering
\begin{tabular}{@{}cccc@{}}
\toprule
\textbf{Method/Memory size} & $\mathbf{50}$      & $\mathbf{100}$     & $\mathbf{200}$     \\ \midrule
CL uniform sampling           & $85.23 \pm 1.84$  & $92.80 \pm 0.79$  & $95.08 \pm 0.30$ \\
CL BiCo          & $91.61 \pm 0.78$  &  $95.81 \pm 0.28$  &  $97.01 \pm 0.41$ \\
Streaming reservoir sampling            & $83.90 \pm 3.18$   & $90.72  \pm 0.97$   & $94.12 \pm 0.61$  \\
Streaming BiCo             &  $85.32 \pm 2.40$  & $92.51 \pm 1.30$  & $95.50 \pm 0.65$   \\ \bottomrule
\end{tabular}
    \caption {Replay memory size study on SMNIST. Our method offers bigger improvements with smaller memory sizes. }
    \label{table:buffer-size-study}
\end{table}

\looseness -1 In our experiments, we used the Neural Tangent Kernel as proxy. It turns out that on the data sets derived from MNIST, simpler kernels such as  RBF are also good proxy choices. To illustrate this, we repeat the continual learning and streaming experiments and report the results in Table \ref{table:rbf}. For the RBF kernel $k(x,y) = \exp(-\gamma\normp{x-y}{2}^2)$ we set $\gamma=5 \cdot 10^{-4}$. While the RBF kernel is a good proxy for these data sets, it fails on harder data sets such as CIFAR-10.

\begin{table*}[h]
    \centering
\begin{tabular}{@{}cccc@{}}
        \toprule
        &  \textbf{Method}                       & \textbf{PMNIST}                   & \textbf{SMNIST}                  \\ \midrule
        \multirow{2}{*}{\STAB{\rotatebox[origin=c]{90}{\large{CL}}}} 
     & BiCo CNTK                                    & $79.33 \pm 0.51$  & $95.81 \pm 0.28$   \\
        &  BiCo RBF                                   & $79.95 \pm 0.81$  & $96.09 \pm 0.32$  \\  
        \midrule
        \multirow{2}{*}{\STAB{\rotatebox[origin=c]{90}{\large{VCL}}}} 
        &  BiCo CNTK                                   & $86.11 \pm 0.25$ & $84.62 \pm 0.89$  \\
         & BiCo RBF & $86.16 \pm 0.25$ & $82.21 \pm 1.35$   \\
         \midrule
        \multirow{2}{*}{\STAB{\rotatebox[origin=c]{90}{\large{Str.}}}} 
        & BiCo CNTK                                    & $74.49 \pm 0.69$  & $92.57 \pm 1.09$  \\
        & BiCo RBF                                   & $75.85  \pm 0.65$ & $92.49  \pm 0.71$  \\
       \bottomrule
        \end{tabular}
    \caption {RBF vs CNTK proxies.}  \label{table:rbf}
        \vspace{-4mm}
\end{table*}
    
\vskip 0.2in
\bibliography{bibliography}

\end{document}